\documentclass[runningheads]{llncs}

\usepackage{eccv}

\usepackage{eccvabbrv}

\usepackage{graphicx}
\usepackage{booktabs}

\usepackage[accsupp]{axessibility}  

\usepackage{algorithm}
\usepackage{algpseudocode}
\usepackage{makecell}
\usepackage{orcidlink}
\usepackage{bm}

\begin{document}

\title{Lost in Translation: Modern Neural Networks Still Struggle With Small Realistic Image Transformations}

\titlerunning{Modern NNs Struggle With Small Realistic Image Transformations}

\author{Ofir Shifman \and Yair Weiss
}


\authorrunning{Shifman and Weiss}

\institute{The Hebrew University of Jerusalem \\
\email{\{Ofir.Shifman, Yair.Weiss\}@mail.huji.ac.il}}

\maketitle
\begin{abstract}
Deep neural networks that achieve remarkable performance in image classification have previously been shown to be easily fooled by tiny transformations such as a one pixel translation of the input image. In order to address this problem, two approaches have been proposed in recent years. The first approach suggests using huge datasets together with data augmentation in the hope that a highly varied training set will teach the network to learn to be invariant. The second approach suggests using architectural modifications based on sampling theory to deal explicitly with image translations.  In this paper, we show that these approaches still fall short in robustly handling 'natural' image translations that simulate a subtle change in camera orientation. Our findings reveal that a mere one-pixel translation can result in a significant change in the predicted image representation for approximately 40\% of the test images in state-of-the-art models (e.g. open-CLIP trained on LAION-2B or DINO-v2) , while models that are explicitly constructed to be robust to cyclic translations can still be fooled  with 1 pixel realistic (non-cyclic) translations 11\% of the time. We present \textbf{R}obust \textbf{I}nference by \textbf{C}rop \textbf{S}election: a simple method that can be proven to achieve any desired level of consistency, although with a modest tradeoff with the model's accuracy. Importantly, we demonstrate how employing this method reduces the ability to fool state-of-the-art models with a 1 pixel translation to less than 5\%  while suffering from only a 1\% drop in classification accuracy.  Additionally, we show that our method can be easy adjusted to deal with circular shifts as well. In such case we achieve 100\% robustness to integer shifts with \textit{state-of-the-art} accuracy, and with no need for any further training.


\end{abstract}    
\section{Introduction}
\label{sec:intro}

\begin{figure}[t]
  \centering
  \begin{subfigure}{1\linewidth}  
    \fbox{\includegraphics[width=0.95\linewidth]{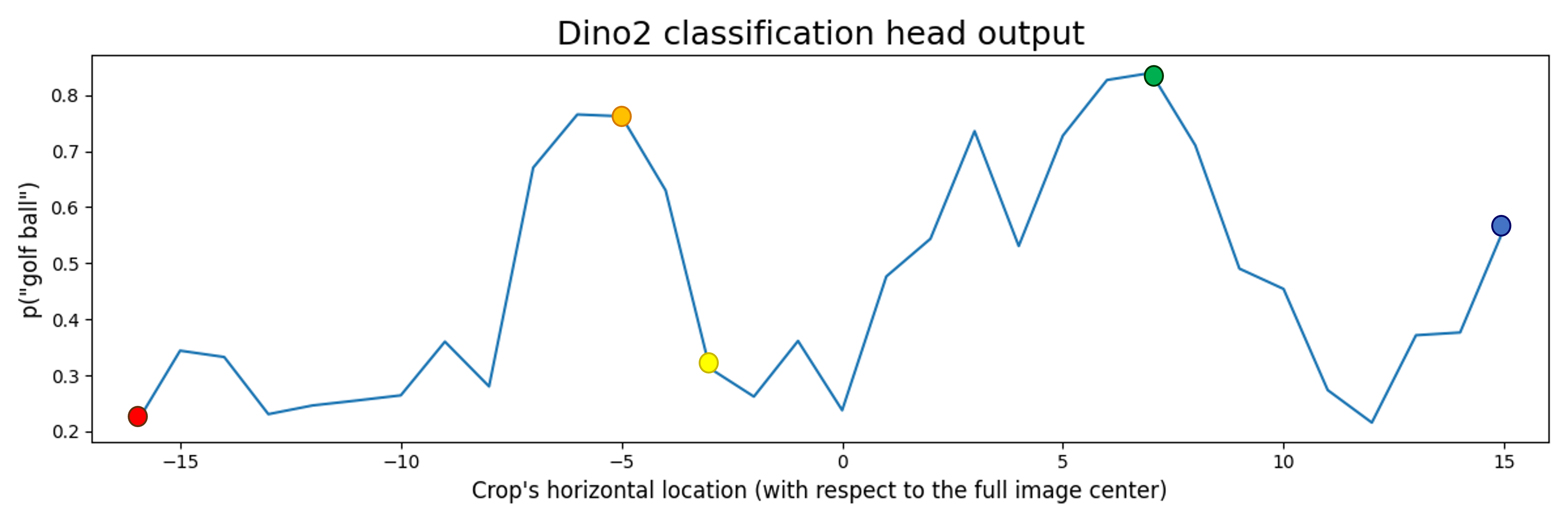}}
    \caption{DINOv2's output score for the true label ("golf ball") as a function of the crop location.}
    \label{fig:p_golf_ball}
  \end{subfigure}
  \hfill
  \begin{subfigure}{1\linewidth}
    \fbox{\includegraphics[width=0.95\linewidth]{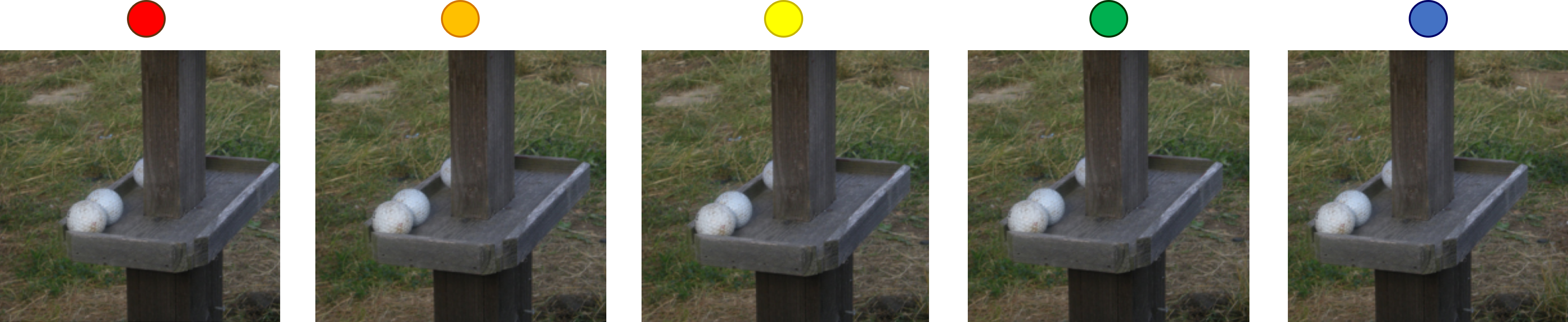}}
    \caption{Five Examples of different crops, the colored circles connect the image to it's predicted probability}
    \label{fig:golf_ball_images}
  \end{subfigure}
  \caption{The modern DNN DINOv2 \cite{oquab2023dinov2} still suffers from significant changes in the output probability for the true label ("golf ball") as a function of minor realistic translations. (a) shows DINOV2's probability for "golf ball" for 32 crops, each measuring 224x224 pixels, extracted from the same 256x256 pixels image  with lateral translation only. (b) Demonstrates that this translation by a few pixels translation is nearly imperceptible.   We limit our experiments to evaluation set in which the true label object remains fully observable within all the assessed crops, as can be seen here.}
  \label{fig:golf_ball}
\end{figure}

\begin{figure}[htb]
  \centering
  \begin{subfigure}{0.58\linewidth}
    \centering
    \includegraphics[width=\linewidth]{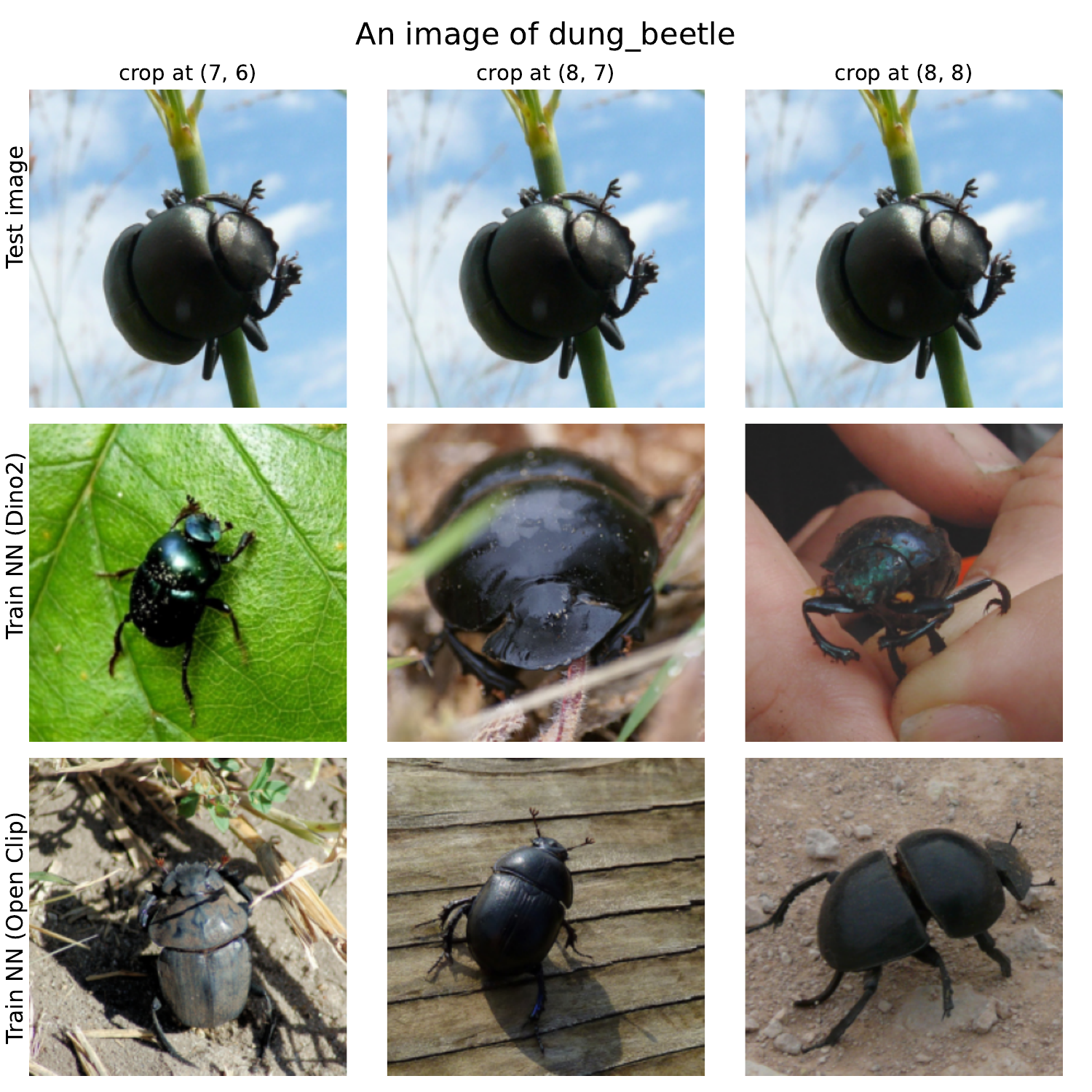}
    \caption{\textbf{First row:} 3 different images cropped out from the same image and differ from each other by 1 pixel shifts. \textbf{Second row:} The retrieved images that are closest in terms of DINOv2's representation to the query image. The search is performed over all 1300 images in the same class.  \textbf{Third row:} The nearest neighbor of the top images using Open Clip representation. }
    \label{subfig:beetles}
  \end{subfigure}\hfill
  \begin{subfigure}{0.38\linewidth}
    \centering
    \includegraphics[width=\linewidth]{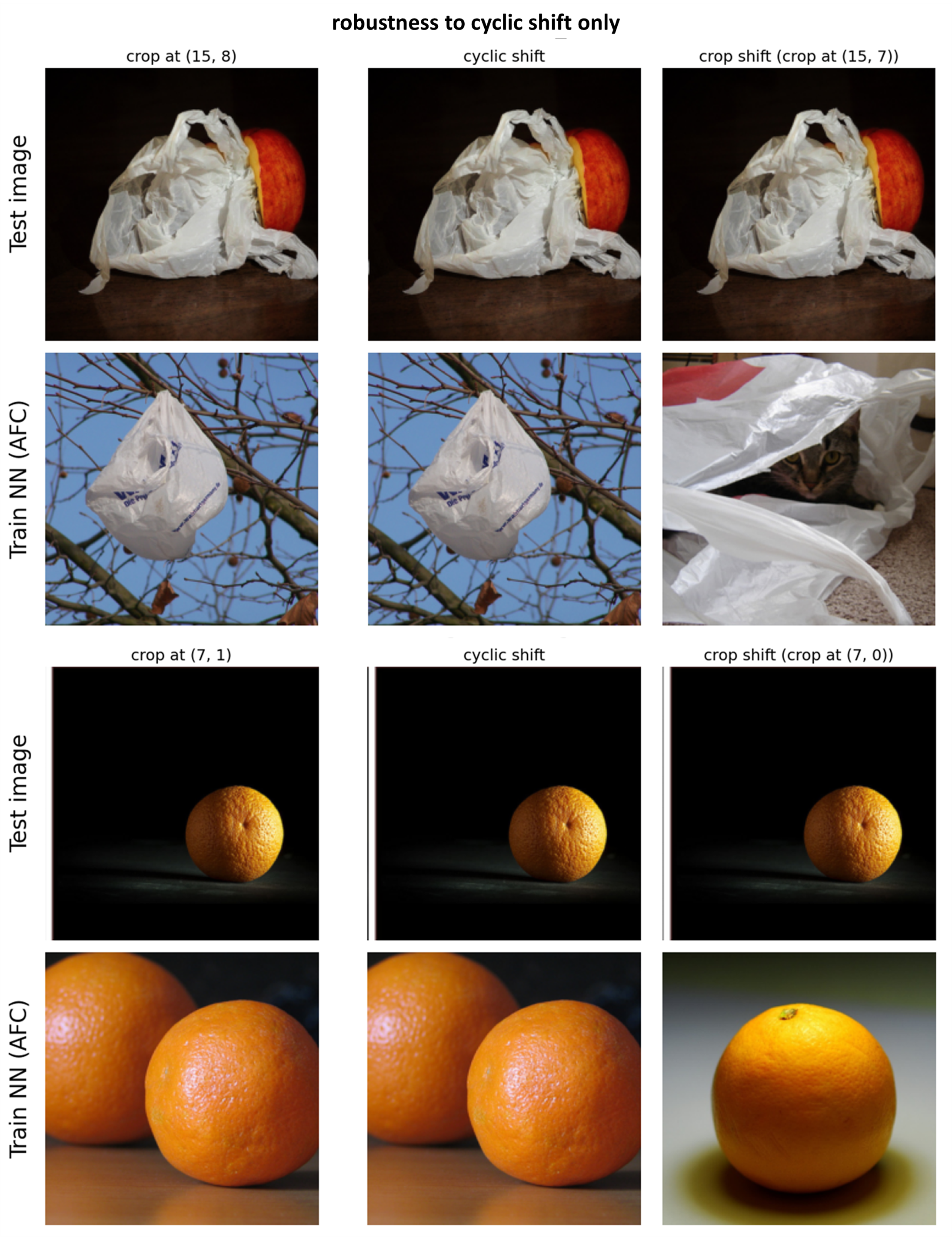}
    \caption{\textbf{First \& Third rows (left to right)}: A test image, the image translated with cyclic shift, and with a realistic shift. One might have hard time seeing the difference since in both cases the shift is by  a single pixel. \textbf{Second \& Forth rows:} The images from the training set whose AFC representation is the nearest neighbor of the top images' representations. }
    \label{subfig:afc}
  \end{subfigure}
  \caption{Modern Neural Networks still suffer from inconsistency to minor translations: (a) shows an example in which both DINO-V2 and Open-CLIP fail to be consistent on these indiscernible translations. While (b) shows that a method that achieves 100\% consistency to cyclic shifts\cite{michaeli2023aliasfree} still suffers from lack of robustness to imperceptible realistic shifts.}
  \label{fig:modern_fails}
\end{figure}

Neural networks have shown remarkable performance in  image classification and have even been described as achieving "super-human" performance \cite{he2015delving}. One reason for the success of Convolutional Neural Networks (CNNs) is the inductive bias that arises from the use of convolutions and pooling~\cite{fukushima1983neocognitron,lecun1989backpropagation} which should in principle make them robust to small perturbations of the image and mainly to image-translations.

Despite their reliance on convolutions, in the last five years it has been shown repeatedly that state-of-the-art deep CNNs are {\em not} invariant to tiny translations of the input \cite{azulay2018deep,zhang2019making}, and can often change their output as a result of a simple translation of the image by a single pixel. Importantly, recent works suggest that vision transformers  also exhibit  vulnerability to such transformations\cite{gunasekar2022generalization, rojas2023making, wu2021cvt}. These results add to the growing literature on the lack of robustness of deep neural networks to small changes in the input or to test situations that are different from the training scenario~\cite{hendrycks2019benchmarking,hendrycks2016baseline,recht2019imagenet,geirhos2020shortcut,wortsman2022robust}.

In order to make modern classifiers robust to tiny perturbations, two types of solutions have been proposed. 
The first approach suggests using huge datasets possibly in combination with self-supervised pretraining and data augmentation assuming that the use of a large and varied dataset will cause the classifier to learn to be robust\cite{gunasekar2022generalization, radford2021learning, schuhmann2022laion, rebuffi2021data}. For example,  Gunasekar \cite{gunasekar2022generalization} used data augmentation to increase robustness to translations and argued that "In ultra-large scale datasets,
accuracy/robustness might naturally come from dataset size itself rather than model priors."
This approach was also advocated in Radford et al. ~\cite{radford2021learning} which argued that "large-scale task and dataset agnostic pre-training combined with a reorientation towards zero-shot and fewshot benchmarking on broad evaluation suites (as advocated by Yogatama et al. ~\cite{yogatama2019learning} and Linzen ~\cite{linzen2020can}) promotes the development of more robust systems"~\cite{radford2021learning}.  
This "robustness by large-scale" approach often goes together with the use of the popular vision-transformer (ViT) architecture \cite{dosovitskiy2020image,caron2021emerging,naseer2021intriguing}, which is usually pre-trained on a massive dataset of up to five billion images \cite{schuhmann2022laion, oquab2023dinov2}. 

A second approach uses sampling theory to modify CNNs so that they will be invariant to translation. In particular, this approach is based on the observation that standard  architectures may violate the Shannon-Nyquist sampling theorem~\cite{zhang2019making, michaeli2023aliasfree,zou2023delving}. Several recent papers~\cite{chaman2021truly,michaeli2023aliasfree,rojas2022learnable,rahman2023truly} have shown how to modify CNN architecture so that the network is always robust to any {\em circular (i.e. cyclic)} translation of the input: when the image is translated to the right, the rightmost column reappears as the leftmost column in the translated image. Similar ideas have also been extended to  ViT architectures and have made it possible to make them robust to {\em circular} translation as well \cite{rojas2023making}.


In this paper, we show that both of these existing approaches to providing robustness still fall short of attaining translation invariance to simple, realistic perturbations. When a camera is rotated by a small angle along its optical center, the image is indeed shifted by a small translation. But this translation is {\em not} cyclic. Rather a small part of the scene exits the field of view in one side and another part of the scene enters the field of view in the other side. In image-processing terms, this corresponds to taking a large field of view image, and performing different crops that are slightly smaller. The bottom of figure~\ref{fig:golf_ball_images} illustrates what we call "realistic" image translations: each image shown on the bottom is a slightly different crop of a larger image. All the images are translations of each other, but they are not circular translations. Such translations have also been referred to as "crop-shifts"~\cite{michaeli2023aliasfree} or "standard-shift" \cite{rojas2022learnable, rojas2023making}. 

How well do the suggested approaches to providing robustness work for these kind of realistic translations? As can be shown in \cref{fig:golf_ball,fig:modern_fails} they still fall short of providing translation invariance even for tiny translations. Figure~\ref{fig:p_golf_ball} shows that the output of DINOv2 \cite{oquab2023dinov2} which was trained on billions of images in a self-supervised manner, still changes drastically when the image is shifted by a few pixels. ~\cref{subfig:beetles} shows a similar effect for both DINOv2 and OpenCLIP, and \cref{subfig:afc} shows a similar failure with AFC ~\cite{michaeli2023aliasfree},  a recent method which was designed to be invariant to cyclic shifts, but also shows better robustness to "realistic" or "crop-shift" translations, than any other previous works.

Modern models often serve as foundational models for diverse tasks \cite{bommasani2021opportunities, awais2023foundational, yuan2021florence}. Given that downstream tasks rely on these models' computed representations, maintaining consistency of the representation becomes vital. As can be seen in ~\cref{fig:modern_fails}, we evaluate the robustness of the feature representation by using it to retrieve nearest neighbors from a set of 1300 images from the same class. For the AFC method, when the images are translated by a single pixel using {\em cyclic} shift, the networks output stays the same as the retrieved neighbor does not change. However when it is translated using a {\em realistic} one-pixel translation (i.e. the first column of pixels disappears in one side and a new column of pixels is introduced in the other side), then the retrieved image changes drastically even though the images are perceptually identical.




While these results are meant to illustrate individual failures of existing approaches, we present here quantitative experiments that show that these failures are surprisingly common. 
We also present a simple method that can be proven to achieve any desired level of consistency, although with a modest tradeoff with the model's accuracy. Importantly, we demonstrate how employing this method reduces the ability to fool state-of-the-art classifiers  to less than 5\% while suffering from approximately a 1\% drop in classification accuracy. Additionally, we show that our method can be easy adjusted to deal with circular shifts as well. In such case we achieve 100\% robustness to integer shifts, as was previously reported in~\cite{chaman2021truly, michaeli2023aliasfree, rojas2022learnable, rojas2023making}, but with \textit{state-of-the-art} accuracy, and with no need for any training at all.

\section{Defining and Measuring Robustness to Translations}

We use two measures to quantify the robustness of a classifier to a set of transformations. The first one is referred to as "Consistency"  \cite{azulay2018deep, chaman2021truly, michaeli2023aliasfree} and measures the probability that a random shift in the set of allowed translations (e.g. a 1-pixel translation) will change the output of the classifier. The second one, "Adversarial Robustness" was also used by Michaeli et al. \cite{michaeli2023aliasfree} and asks whether an adversary which is only allowed to replace an image with a translation in the set of allowed translations (e.g. 1 pixel shifts) will be able to make the classifier change its output. 

Mathematically, define $T(I,\Delta)$ as the translation of image $I$ by offset $\Delta$ and $f_\theta(I)$ is the output of the classifier, then consistency is defined as 

\begin{equation}
\mathbb{E}_{I,\Delta}\left[f_\theta(I)==f_\theta(T(I,\Delta))\right]
\label{eq:consistency}
\end{equation}
where the expectation is taken over images and translations $\Delta$. Adversarial robustness if defined as 
\begin{equation}
\mathbb{E}_I \left[ \min_{\Delta} 
f_\theta(I)==f_\theta(T(I,\Delta))\right]
\label{eq:adv-rob}
\end{equation}

Here the expectation is taken over images, but the translation $\Delta$ is chosen by the adversary in order to fool the classifier rather than randomly. By definition, adversarial robustness is lower than consistency. 

We evaluated robustness and consistency using two different definitions of the classifier's output $f_\theta(\cdot)$. In the standard ImageNet setting, the output of the classifier is one of 1000 labels and a model is robust if this label does not change with small translations. The problem with this definition is that it measures the robustness of the final classification and not the robustness of the underlying representation computed by the network, which can be used as input for various tasks when working with DNN as a foundation model \cite{bommasani2021opportunities,awais2023foundational,fang2023eva}. Note that if the network's representation is provably robust, then we will always know that the classification will be robust. However, a classifier that is robust on one task, is not guaranteed to calculate a robust representation of an arbitrary image, especially when dealing with data that is {\em out of distribution} \cite{hendrycks2021many}.

Therefore, in our second definition we focus on the representation computed by the network, which is typically a vector in dimension between 512 and 2048.
 A natural approach would be to check if this embedding vector does not change when the image is translated by a single pixel (or changes by a negligible amount). In order to define what constitutes a negligible change in the embedding vector,  we use the vectors to retrieve the most similar image in a dataset of 1300 images (all other images in the same class). Now the robustness measures whether the index of the retrieved image is the same for an original image and a translated image. 

As mentioned in the introduction, we use a definition of translation that is different from the cyclic translations used in many previous papers and captures the effect of a small natural rotation of the camera around its optical center. We believe this definition is more realistic and relevant to most applications. For evaluation we follow the standard protocol in training classifiers on ImageNet: images are resized to size $256 \times 256$ and then a crop of size $224 \times 224$ is taken. An image is translated one pixel to the right by moving the center of the crop one pixel to the right. This means that the column of pixels at the left hand side of the image disappears from the image, while a new column of pixels appears in the right hand side. All the other columns are shifted to the left.

The top part of ~\cref{tab:results} shows the evaluation of different definitions of robustness for different baseline models and different shift sizes. open-CLIP and DINOv2 are representatives of the "robustness by large scale" approach while AFC serves as the representative of the signal processing approach that is guaranteed to be invariant to cyclic translations. As can be seen, all methods are still far from robust {\em even for one pixel translations} and the best model in terms of accuracy (DINOv2) has a 1-NN consistency of less than 85\% and an adversarial robustness less than 63\% to 1-pixels shifts.

\begin{table}[htbp]
  \centering
  \resizebox{1.05\textwidth}{!}{
    \begin{tabular}{@{}l|c|cccc|cccc|cccc|cccc@{}}
    \toprule
    \multicolumn{1}{c}{Model} & \multicolumn{1}{c}{Accuracy} & \multicolumn{4}{c}{\makecell[c]{Adv-Rob (1-NN)\\{\footnotesize Shift Size:}}} & \multicolumn{4}{c}{\makecell[c]{Adv-Rob (Class)\\{\footnotesize Shift Size:}}} & \multicolumn{4}{c}{\makecell[c]{Consistency (1-NN)\\{\footnotesize Shift Size:}}} & \multicolumn{4}{c}{\makecell[c]{Consistency (Class)\\{\footnotesize Shift Size:}}} \\
     \cmidrule(lr){2-2} \cmidrule(lr){3-6} \cmidrule(lr){7-10} \cmidrule(lr){11-14} \cmidrule(lr){15-18} 
    & & 1  & 3 & 5 & 9 & 1  & 3 & 5 & 9 & 1  & 3 & 5 & 9 & 1  & 3 & 5 & 9 \\
    \midrule
vanila vit   &   76.25*	& 37.21 & 	14.52 & 	8.93 & 	6.40 & 	- & 	- & 	- & 	- & 	71.02 & 	61.90 & 	56.95 & 	52.43 &    - & -&- & -\\
Resnet50   &   76.74 & 	46.33 & 	25.08 & 	17.44 & 	10.39 & 	86.45 & 	78.16 & 	74.08 & 	69.78 & 	75.96 & 	70.23 & 	67.24 & 	63.44 & 	94.50 & 	93.13 & 	92.42 & 	91.60 \\
ConvNext-T   &   78.29 & 	59.16 & 	37.53 & 	29.07 & 	21.20 & 	89.79 & 	82.37 & 	78.77 & 	75.18 & 	83.09 & 	77.54 & 	74.29 & 	70.86 & 	96.10 & 	94.64 & 	93.72 & 	92.79 \\
ConvNext-L   &   \textbf{83.47} & 	\textbf{65.58} & 	45.70 & 	37.35 & 	29.13 & 	94.35 & 	89.82 & 	87.62 & 	85.65 & 	\textbf{86.33} & 	82.15 & 	79.75 & 	77.38 & 	97.92 & 	97.19 & 	96.76 & 	96.33 \\
open-CLIP   &   69.33* & 	57.01 & 	35.28 & 	27.81 & 	22.22 & 	85.37 & 	74.55 & 	69.40 & 	65.25 & 	81.98 & 	76.16 & 	72.99 & 	70.25 & 	94.44 & 	92.53 & 	91.47 & 	90.54 \\
DINOv2   &   \textbf{84.11} & 	\textbf{62.42} & 	40.74 & 	32.68 & 	26.91 & 	93.45 & 	88.50 & 	86.26 & 	84.16 & 	\textbf{84.77} & 	79.26 & 	76.62 & 	74.79 & 	97.60 & 	96.66 & 	96.21 & 	95.87 \\
\midrule
AFC   &   74.91 & 	89.77 & 	76.54 & 	67.35 & 	55.01 & 	97.28 & 	93.67 & 	91.11 & 	87.67 & 	96.46 & 	93.65 & 	91.37 & 	87.80 & 	99.05 & 	98.29 & 	97.64 & 	96.57 \\
\midrule
\textbf{Theoretical Bound}   &   -  &	91.02 & 	75.37 & 	62.33 & 	42.31 & 	91.02 & 	75.37 & 	62.33 & 	42.31 & 	95.40 & 	86.83 & 	79.01 & 	65.37 & 	95.40 & 	86.83 & 	79.01 & 	65.37 \\
R-MH (ConvNext-T)   &   72.87 & 	94.30 & 	83.93 & 	74.62 & 	58.12 & 	97.94 & 	94.01 & 	90.08 & 	82.95 & 	97.85 & 	95.15 & 	92.59 & 	87.80 & 	99.23 & 	98.25 & 	97.23 & 	95.21 \\
R-R (ConvNext-T) &   73.31 & 	93.47 & 	81.54 & 	70.84 & 	52.85 & 	97.70 & 	93.46 & 	89.66 & 	82.63 & 	97.52 & 	94.38 & 	91.39 & 	85.83 & 	99.14 & 	98.06 & 	97.09 & 	95.34 \\
R-MH (ConvNext-L)   &   78.93 & 	94.98 & 	85.58 & 	76.89 & 	61.04 & 	\textbf{98.50} & 	\textbf{95.68} & 	\textbf{92.80} & 	\textbf{87.36} & 	98.11 & 	95.71 & 	93.38 & 	88.82 & 	\textbf{99.44} & 	\textbf{98.75} & 	\textbf{98.05} & 	\textbf{96.54} \\
R-R (ConvNext-L)   &   78.10 & 	94.48 & 	84.18 & 	74.63 & 	58.43 & 	98.28 & 	94.96 & 	91.76 & 	86.20 & 	97.91 & 	95.25 & 	92.70 & 	87.83 & 	99.36 & 	98.53 & 	97.71 & 	96.26 \\
R-R (open-CLIP)   &   67.77* & 	\textbf{95.15} & 	\textbf{86.26} & 	\textbf{77.95} & 	\textbf{63.38} & 	97.91 & 	94.06 & 	90.43 & 	83.75 & 	\textbf{97.97} & 	\textbf{95.39} & 	\textbf{92.87} & 	\textbf{88.19} & 	99.17 & 	98.14 & 	97.23 & 	95.64 \\
R-MH (open-CLIP)   &   65.75* & 	94.63 & 	84.58 & 	75.22 & 	59.09 & 	97.79 & 	93.71 & 	90.07 & 	83.20 & 	98.17 & 	95.90 & 	93.65 & 	89.38 & 	99.21 & 	98.26 & 	97.34 & 	95.54 \\
\textbf{R-R (DINOv2)}   &   \textbf{82.66} & 	\textbf{94.73} & 	\textbf{85.04} & 	\textbf{76.21} & 	\textbf{61.47} & 	\textbf{98.68} & 	\textbf{96.19} & 	\textbf{94.01} & 	\textbf{89.88} & 	\textbf{98.00} & 	\textbf{95.51} & 	\textbf{93.13} & 	\textbf{88.76} & 	\textbf{99.51} & 	\textbf{98.89} & 	\textbf{98.34} & 	\textbf{97.40} \\
\textbf{R-MH (DINOv2)}    &  \textbf{82.43}  & \textbf{ 	95.14 }  & \textbf{ 	86.10 }  & \textbf{ 	77.78 }  & \textbf{ 	63.08 }  & \textbf{ 	98.78 }  & \textbf{ 	96.41 }  & \textbf{ 	94.06 }  & \textbf{ 	89.45 }  & \textbf{ 	98.16 }  & \textbf{ 	95.83 }  & \textbf{ 	93.53 }  & \textbf{ 	89.26 }  & \textbf{ 	99.54 }  & \textbf{ 	98.95 }  & \textbf{ 	98.39 }  & \textbf{ 	97.17} \\
    \bottomrule
    \end{tabular}
  }
  \caption{All models accuracy, adversarial robustness (\ref{eq:adv-rob}) and consistency (\ref{eq:consistency}) in terms of nearest neighbor and of class change, for shift sizes of 1,3,5 and 9 pixels. Accuracy is evaluated with linear classification (*besides open-CLIP and vanila-ViT that are evaluated with K-NN) on a random sample of 1000 images from ImageNet validation set. R-R stands for RICS-rand and R-MH for the use of Mexican-Hat kernel.}
  \label{tab:results}
\end{table}
\newcommand{\R}{\mathbb{R}}
\newcommand{\N}{\mathbb{N}}

\section{Our Method}
\label{sec:Our Method}
\subsection{Robust Inference by Crop Selection}

\begin{figure}[htbp]
  \centering  \fbox{\includegraphics[width=0.95\linewidth]{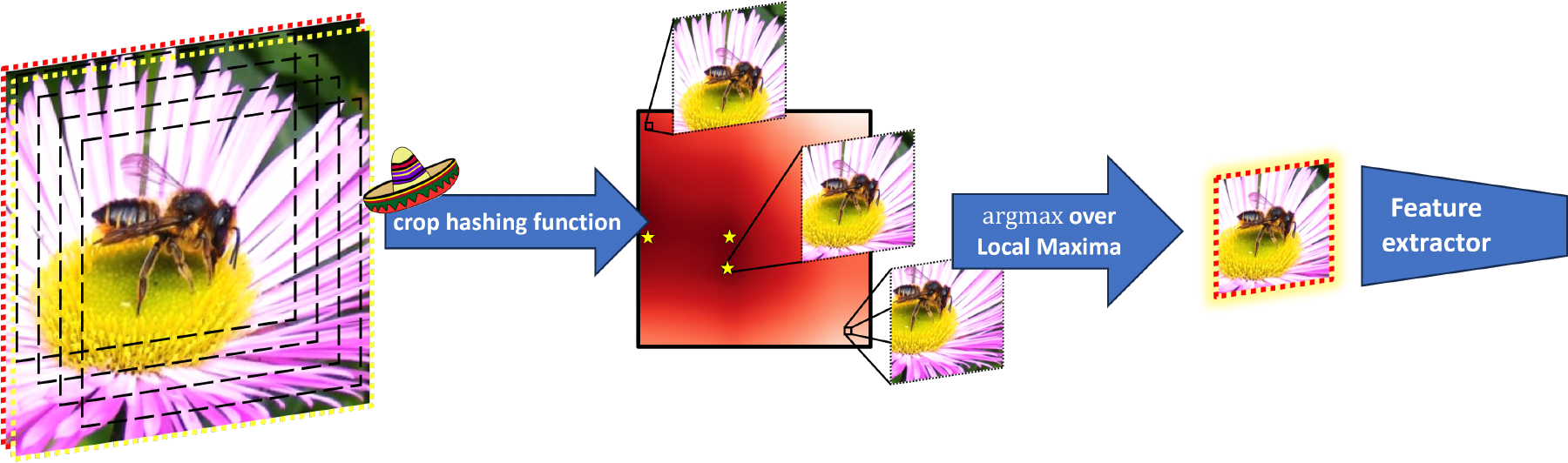}}
  \caption{Our proposed method: Robust Inference by Crop Selection (RICS): First assign scores to all of the different crops of the image with some deterministic function, then choose the crop with the maximal score (only if it's also a local maximum: it has 8 neighbors, and it's score is the highest among them). This crop is passed through a feature extractor. \textbf{Left:} An image and it's translated version, with few crops marked. \textbf{Center:} The crops' scores as a heatmap. local maximums are highlighted with a star. \textbf{Right:} The chosen crop is passed to the feature extractor.}
  \label{fig:algorithm12}
\end{figure}

We present a method that allows us to convert any classifier into one that is robust to realistic image translations {\em without the need for retraining}. The method is based on the classical object recognition paradigm that is standard in computer vision: given an image it is customary to first find the bounding box of the object and then pass the subimage within the bounding box to a classifier (e.g. ~\cite{girshick2014rich}). The problem of image classification is different from that of object detection (e.g. many ImageNet categories are not based on a single object on a uniform background and there is no well-defined object to detect) but we use it as inspiration for our method.

During training of image classifiers, it is common practice to perform data augmentation by using different crops of the same image. Our method is built on the simple observation that we can make a classifier robust to translations by choosing a subcrop of the image during inference, not training. 

Our method, which we call \textbf{Robust Inference by Crop Selection} is illustrated in ~ \cref{fig:algorithm12} and \cref{alg:robust-inference}: Given an image, we first assign scores to all of the different crops of the image and then choose the crop with the maximal score  (among the local maxima, i.e. a score must be larger than the score of its 8 neighboring crops and not at the edge). This crop is then passed through the standard neural network to serve as the feature extractor of the image, and the output of our method on the image is defined as the output to the specific crop.

At first glance, our method seems to require a complicated score function that will include exactly the relevant parts of the image. But in fact, in order to ensure robustness, all we require is that the crop be chosen in a consistent manner. Consider again figure~\ref{fig:algorithm12} and two images of the bee that differ by a single pixel translation. In order for our method to yield consistent classification, we only require that the same crop be chosen for the two images but we do not require that the crop be exactly centered on the bee.

In our implementation we use two score functions. One which we call {\bf RICS-rand} (R-R) involves performing a dot product of the crop with a fixed, random, filter (each pixel in the filter is chosen IID from a Gaussian, and the filter is constant for all images and crops). The output of the dot-product is then passed through a  modulo function. We use this score function to illustrate the fact that our score function is not required to be a good object detector. A dot product with a random filter is certainly not a good method to choose the crop most likely to contain the object at the center, but it is still deterministic: the score only depends on the pixel values within a crop and as we show in our theoretical analysis, a deterministic score function that is pseudo-random is enough to guarantee robustness of our method.

The second function is more closely related to classical methods for object recognition but still very simplistic. We perform a dot-product with a "Mexican Hat" filter \cite{reddy2014object, sarvaiya2011automatic} and
after calculating the scores of all patches, return the crop that has the maximal score, and is also a local maximum (has a larger score than that of the 8 neighbor crops). We start with a Mexican Hat kernel with high standard deviation (50 pixels), that should detect big region of interests, and if this returns scores that are very smooth, with no local maximum (besides the edge ones), we use a smaller kernel (20 or even 9 pixels). We call this method \textbf{RICS-MH} (R-MH). A formal definition of our method follows.

\begin{algorithm}[H]
\caption{Robust Inference by Crop Selection}\label{alg:robust-inference}
\begin{algorithmic}[1]
\linespread{1.2}\selectfont
    \Statex \textbf{Input:} An image $\bm{I \in \mathbb{R}^{n \times n}}$, crop size $\bm{k \times k}$, scoring function $\bm{g: \R^{k \times k}\rightarrow \R}$, and a classifier $\bm{f_\theta: \R^{k \times k}\rightarrow \R^d}$
    \Statex \textbf{Output:} Feature representation, $\bm{y \in \mathbb{R}^d}$.
    \State let $G \in R^{{(n-k+1)}^2}$ be a matrix of the crop scores $G_{i,j} = g(C_{i,j})$ for every crop $0 \leq i,j \leq n-k $. 
    \State perform Non Maximum Suppression on $G$
    \State let $C = \arg\max_{i,j} G_{i,j}$ be the crop who's score is the maximal.
    \State \textbf{return} $f_\theta(C)$
\end{algorithmic}
\end{algorithm}

\subsection{Theory}

As mentioned, our method is based on the observation that we do not need to find the "best" crop within an image before running the network. We now show that even a random score function will guarantee a high degree of robustness.

\begin{definition}
(Pseudo-Random Crop Function): a deterministic score function for a crop is called a "Psuedo-random crop function" if for any image, the maximum of the score function within this image is equally likely to be at any crop of the image. 
\end{definition}

If the crop score function were simply chosen randomly for each crop, then it would trivially satisfy the requirement that the maximum is equally likely to be at any crop. But we also require the function to be a deterministic function of the pixels in the crop (similar to random hash functions). Assuming we can find such a crop function, we can guarantee robustness. 

\begin{figure}[hb!]
    \centering
    \fbox{\includegraphics[width=0.95\linewidth]{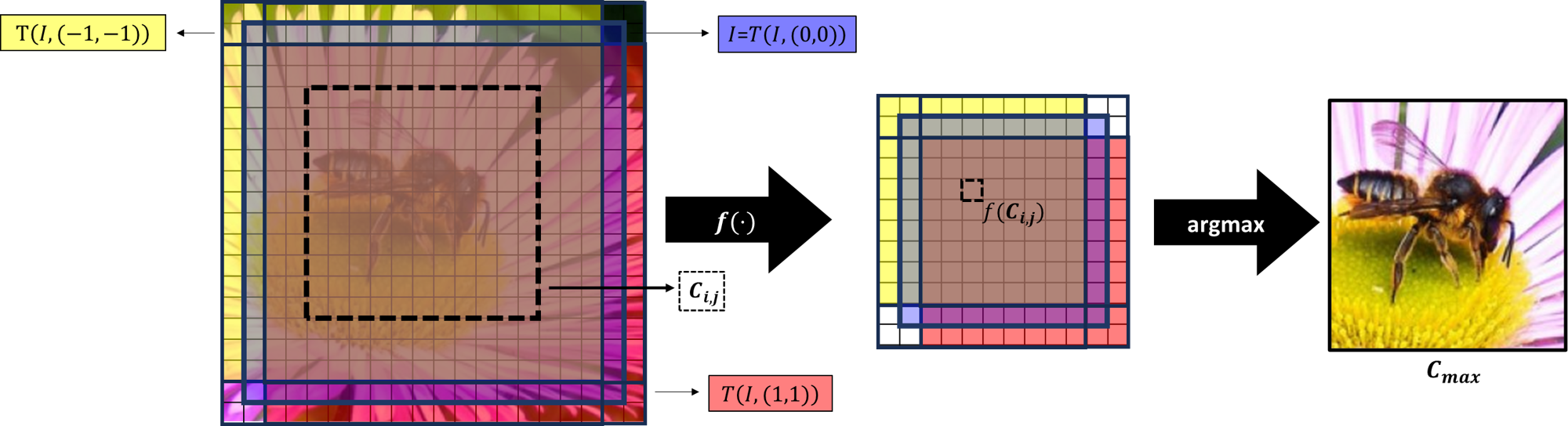}}
    \caption{Illustration of the proof of Theorem 1. If the image is $n\times n$ pixels, and the crop is $k \times k$, than from a total of $(n-k+3)^2$ different crops that correspond to 1 pixel translations, $(n-k-1)^2$ are shared, meaning the probability that the algorithm predicts the same crop for the image and all of its translations is $\left( \frac{n-k-1}{n-k+3}\right)^2$. 
    On the left side you can see the image $I$ (in \textcolor{blue}{blue}) and two versions of it translated by a single pixel colored in \textcolor{yellow}{yellow} and \textcolor{red}{red}. $C_{i,j}$ is the crop that gets the highest score and returned by algorithm \ref{alg:robust-inference}.}
    \label{fig:score_cdf}
\end{figure}

\begin{theorem}
Consider a base classifier that classifies $n \times n$ images. Suppose that we use the RICS method (algorithm 1) with $k \times k$ crops and a pseudo-random score function to convert the base classifier into a new classifier, then the probability that an adversary can alter the output of the new classifier by a realistic translation of size at most 
$\Delta$ pixels is bounded by:

\begin{equation}
p(\text{"adversary\_succeeds"}) \leq 1- \left(\frac{n-k+1-2 \Delta}{n-k+1+2 \Delta}\right)^2
\end{equation}
\end{theorem}

\begin{proof} The proof is based on a counting argument and is illustrated in figure~\ref{fig:score_cdf}. If the crop that is selected in algorithm 1 is the same for all translations of the image, then the adversary will surely fail. Since algorithm 1 chooses a crop by choosing the maximum of a pseudo-random score function, the maximum is distributed uniformly over the image. In order for all translations to choose the same crop, the maximum should not be at the edge of the image, as shown in figure~\ref{fig:score_cdf} and the probability that this happens is given by the ratio of the areas of the center and the area of the full image. 
\end{proof}
For concreteness, suppose we are classifying $256 \times 256$ images with a crop size of $150 \times 150$, then theorem 1 guarantees that {\em any classifier of $150 \times 150$ images can be converted into a classifier whose adversarial robustness to 1 pixel translations is at least $93\%$} using the RICS method. Note that the result is for realistic translations: we are not assuming cyclical translations as in~\cite{chaman2021truly,michaeli2023aliasfree} but rather our theoretical result holds for translations where information exits and enters the field of view. Note also that even though our theorem is stated in terms of adversarial accuracy, a similar result can be proven for consistency as well: any classifier can be converted into a method that satisfies a rigorous lower bound on the consistency with realistic transformations, and from similar counting argument we deduce a lower bound on the consistency:
\begin{equation}
p("consistent\_prediction") \geq \left(\frac{n-k+1- \Delta}{n-k+1+ \Delta}\right)^2
\end{equation}

Figure~\ref{fig:lower_bound-a} shows the lower bound on the adversarial robustness to translations of any classifier when using our RICS method as a function of the  size of the crop $k$, starting with an image of size $256 \times 256$ pixels. Note that this is a lower bound, and depending on the base classifier, the actual robustness may be much higher. This is because even if the adversary succeeded in changing the crop that is given to the classifier by algorithm~\ref{alg:robust-inference}, a good classifier will often output the same output to different crops and thus the adversary will fail to modify the classifier's output. 

\begin{figure}[h]
  \begin{subfigure}{0.5\linewidth}
    \includegraphics[width=1\linewidth]{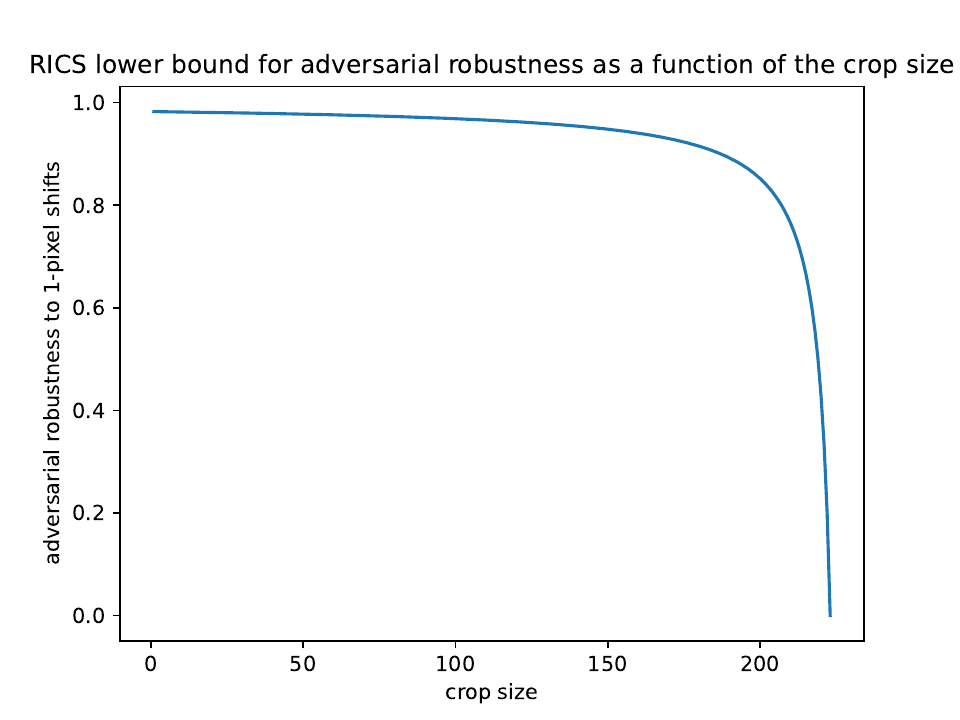}
    \caption{RICS lower bound for the adversarial robustness}
    \label{fig:lower_bound-a}
  \end{subfigure}
  \hfill
  \begin{subfigure}{0.5\linewidth}
   \includegraphics[width=1\linewidth]{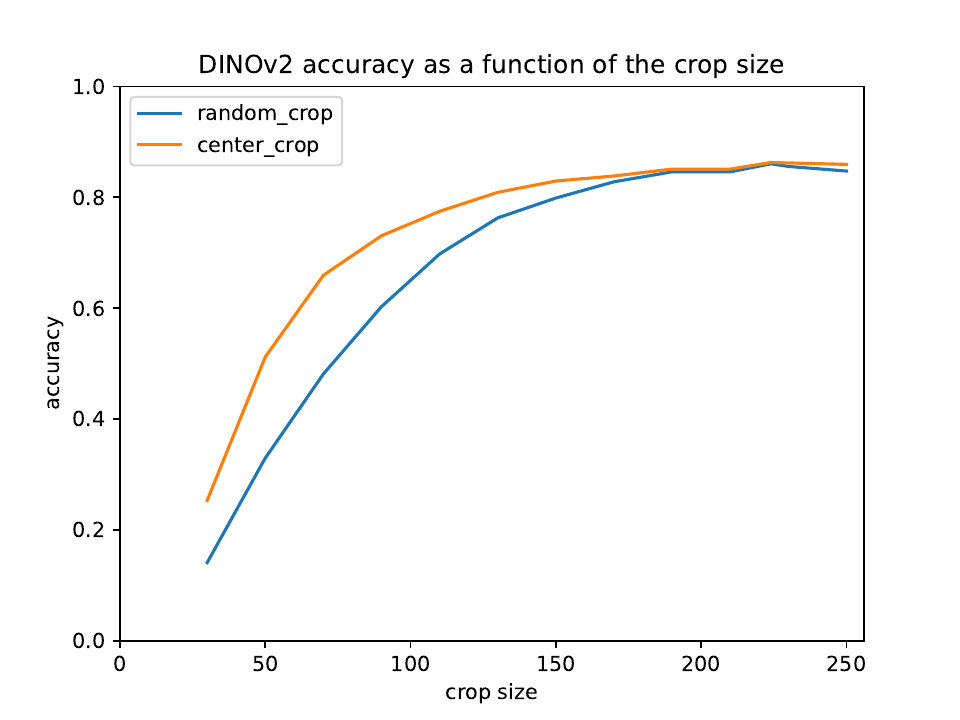}
    \caption{Accuracy of DINOv2 as a function of the crop size}
    \label{fig:lower_bound-acc}
  \end{subfigure}
  \caption{(a) RICS lower bound for the adversarial robustness  when using images of size $224 \times 224$ and shift size of $\Delta = 1$. As the crop size increases, robustness goes down. (b)The accuracy of DINOv2 for different crop sizes. Accuracy goes up as a the crop size increases and more relevant information exist in the image, and center crops are only slightly better than random crops.}
  \label{fig:lower_bound}
\end{figure}

Figure~\ref{fig:lower_bound-acc} shows that accuracy increases as the crop size increases and that random crops are only slightly worse than center crops. The fact that accuracy increases but robustness decreases with crop size means that there is a tradeoff between accuracy and robustness.  In practice, different applications will choose different crop sizes depending on how they want to address this tradeoff.

\subsection{Cyclic Translations}
As mentioned in \cref{sec:intro}, previous works have already achieved 100\% robustness to cyclic translations using polyphase components \cite{chaman2021truly, rojas2022learnable, rojas2023making} or polynomial activations \cite{michaeli2023aliasfree}.
Despite our focus on realistic (non-cyclic) translations, our method can be easily adapted to cyclic shifts as well, by including cyclic crops, ensuring 100\% robustness as well. The fact that our method require no further training, and can be applied as a pre-process to a state-of-the-art model, allows us to achieve state-of-the-art accuracy as well as 100\% robustness to cyclic-translation. The only change in the algorithm is  that we replace the definition of crops in~\cref{alg:robust-inference} with cyclic crops:

\begin{definition}
    A cyclic crop of size $k \times k$ at location $i,j$ of an image $I$ is defined as the top left subimage of $I$ after it has been {\em cyclically}  translated so that location $i,j$ is at the top left. 
\end{definition}

\begin{theorem}
Consider a base classifier that classifies $n \times n$ images.  For any crop size $k$ and any deterministic scoring function, consider using algorithm~\ref{alg:robust-inference} with cyclic crops to convert the base classifier into a new classifier. Then an adversary can not alter the output of the new classifier by any integer cyclic translation. In other words, the adversarial robustness (and therefore also the consistency) of the  new classifier is 100\%.
\end{theorem}

\begin{proof}
For any crop size $k$, the set of cyclic crops of an image $I$ and the set of cyclic crops of the translated version $T_{cyclic}(I, \Delta)$ is exactly the same regardless of the translation size $\Delta$ (as long as the translation is in full pixel values, and not sub-pixel translation), and consist of $n^2$ crops. Since we evaluate the score of all possible crops and choose deterministically, then \cref{alg:robust-inference} will always choose the same cyclic crop, and hence the prediction of the classifier will be the same.
\end{proof}

Note that introducing cyclic crops means that a crop chosen by the algorithm may be one that is unnatural: it may include for example pixels from the bottom of the image that become adjacent to pixels from the top of the image and in principle, our method may decrease the accuracy of the base classifier. Nevertheless, for any classifier, our method is guaranteed to return a new classifier that is perfectly robust to cyclic translations. In the next section we evaluate how much accuracy is decreased as a result of using cyclic crops.

\label{sec:Results}
\begin{figure}[h!]
  \centering
  \begin{subfigure}{0.52\linewidth}
    \centering
    \includegraphics[width=\linewidth]{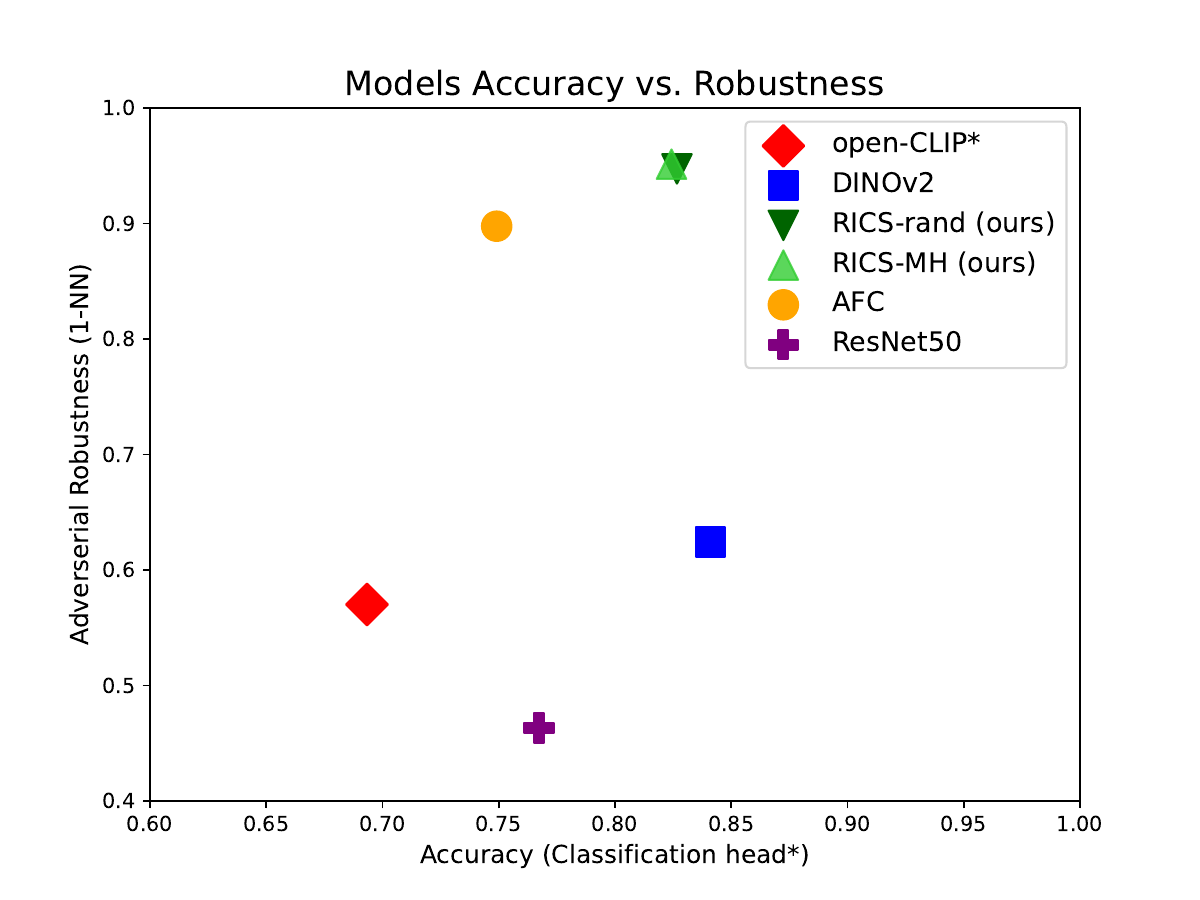}
    \label{subfig:acc_cls_rob}
  \end{subfigure}\hfill
  \begin{subfigure}{0.47\linewidth}
    \centering
    \includegraphics[width=\linewidth]{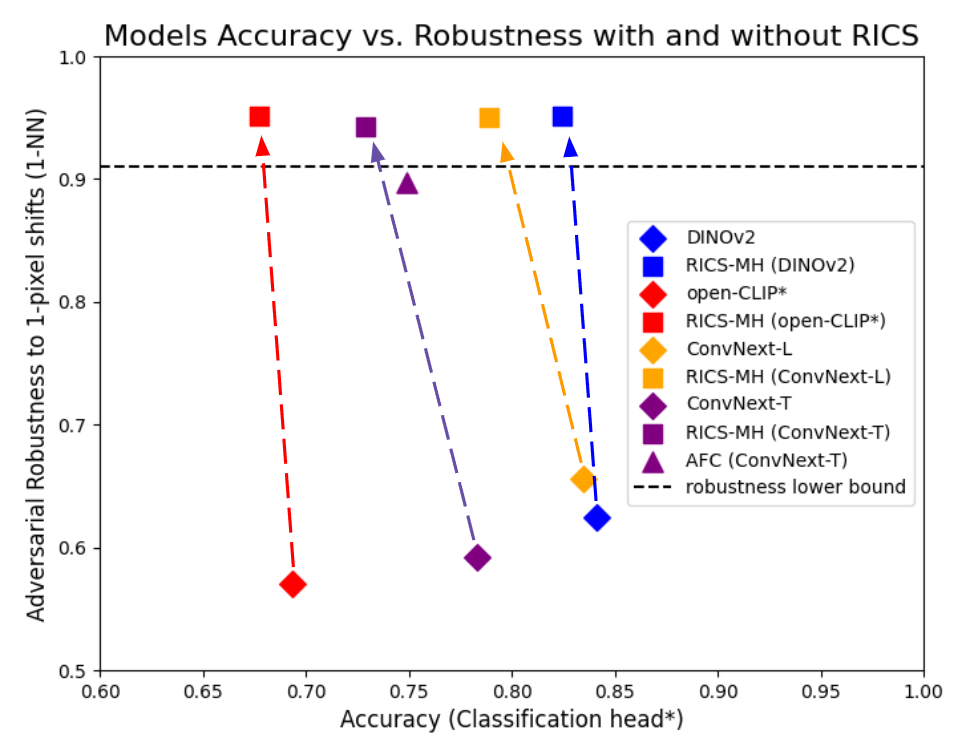}
    \label{subfig:acc_rob_rebuttal}
  \end{subfigure}
  \caption{The accuracy-robustness trade-off. The left plot shows that our method achieves the best robustness compared to exiting methods, while barely losing accuracy comparing to the base model we use, DINOv2. The right plot illustrates that our method can be used to improve the robustness of any model without retraining. The adversarial robustness is based on change of the first Nearest neighbor with 1 pixel shifts.}
  \label{fig:acc_cls_rob}
\end{figure}

\section{Results}

\begin{figure}[b!]
  \centering
  \begin{subfigure}{0.48\linewidth}
    \includegraphics[width=\linewidth]{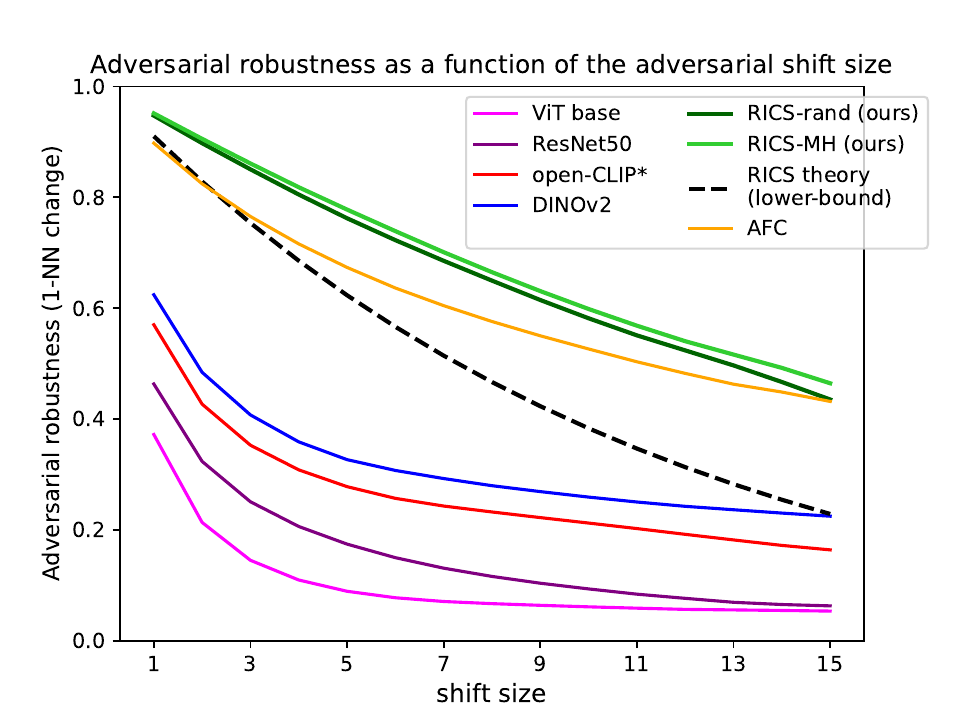}
    \caption{Adversarial robustness as the change of the 1-NN}
    \label{fig:rob_1NN}
  \end{subfigure}%
  \begin{subfigure}{0.48\linewidth}
    \includegraphics[width=\linewidth]{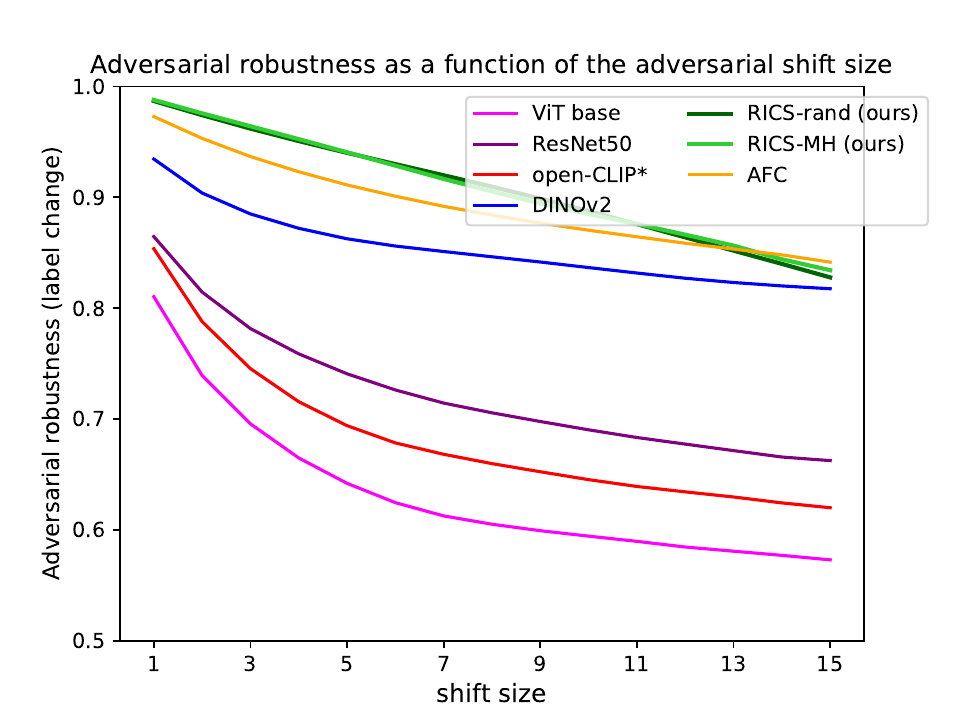}
    \caption{Adversarial robustness as the change of the class}
    \label{fig:rob_cls}
  \end{subfigure}
  
  \begin{subfigure}{0.48\linewidth}
    \includegraphics[width=\linewidth]{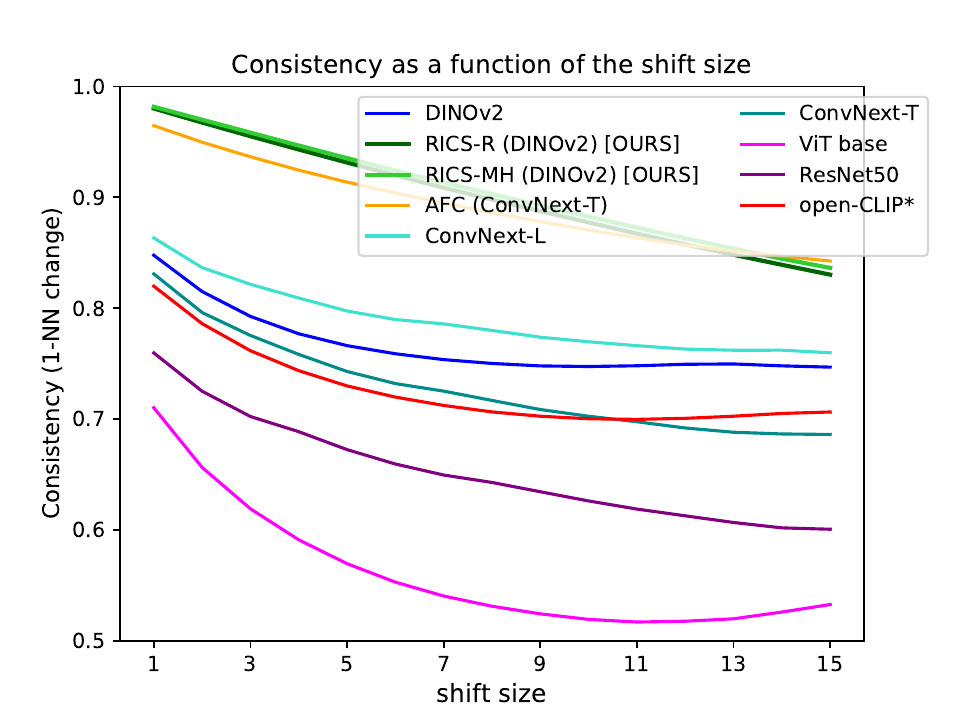}
    \caption{Consistency of the nearest neighbor}
    \label{fig:cons_1NN}
  \end{subfigure}%
  \begin{subfigure}{0.48\linewidth}
    \includegraphics[width=\linewidth]{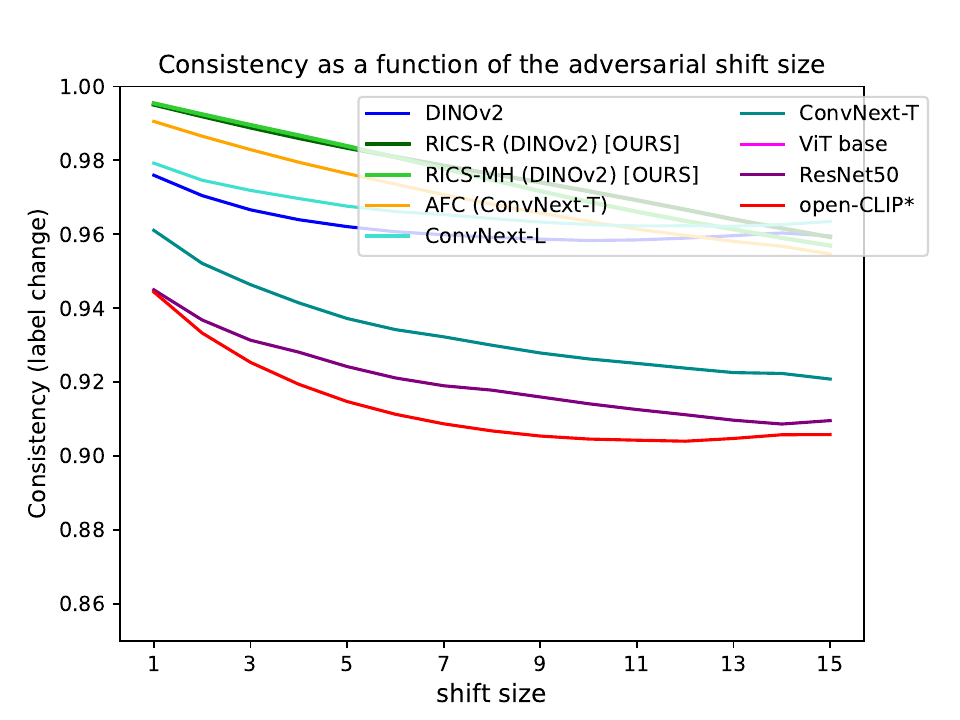}
    \caption{Consistency of the predicted class}
    \label{fig:cons_cls}
  \end{subfigure}

  \caption{Adversarial robustness \ref{eq:adv-rob} and consistency \ref{eq:consistency} as functions of the shift size. As we allow bigger shifts, all models tend to fail more. Adversarial robustness is depicted as a function of the shift size. Subplot (a) illustrates that not only our method is more robust than other methods, even the lower bound is. Subplot (b) shows that when measuring robustness in terms of label change, model accuracy contributes to its robustness. However, the overall ranking of the models remains mostly unchanged. Subplots (c) and (d) shows similar trends on the consistency metric. }

  \label{fig:combined}
\end{figure}

We compare our method to two of the best models that have been published as  open code with trained weights: Open-Clip \cite{schuhmann2022laion} (trained on 2 billion image-text pairs) and DINOv2 \cite{oquab2023dinov2}, as well as AFC~\cite{michaeli2023aliasfree}, the latest version of a model that is constructed to be translation invariant to cyclic shifts. We also compare to three baseline methods: ConvNext \cite{liu2022convnet}, ResNet50 \cite{he2016deep} and vanila-ViT \cite{wu2020visual} trained on ImageNet in the regular supervised manner.

For the evaluation with realistic translations we use a subset of ImageNet consisting of 1000 images, chosen randomly from all images whose object is known to be in all of the $224 \times 224$ pixel crops (after resizing the image to $256\times 256$). We used these images because we wanted to ensure that all the images that we create by small translations still include the main object and to avoid cases where the some translations push the object out of the image. \cite{ILSVRC15}

In order to enable comparison to previous work, we apply our method as if the input is a $224 \times 224$ image. For the realistic translations, we choose the crop size $k$ in our method to be $140 \times 140$ because we have found that this leads to a very small decrease in accuracy while giving high robustness (as can bee seen in Figure ~\ref{fig:lower_bound}). In practice, ImageNet images are usually resized to $256 \times 256$ and then a $224 \times 224$ crop is selected. If our method was used to select $224 \times 224$ crops, then we would expect no decrease in accuracy compared to standard classifiers.

\cref{fig:acc_cls_rob} shows that our method can be applied to any classifier, but our best results are obtained with a baseline DINO-v2. We use both a random filter and a Mexican hat filter. As can be seen in ~\cref{fig:acc_cls_rob} and \cref{tab:results}, our method is the only one to achieve top-1 accuracy of over 80\%  with an adversarial robustness of about 95\% to realistic 1 pixel translations. It can also be seen that for any base classifier, our method significantly improves the robustness with a negligible decrease in accuracy. ~\cref{fig:rob_1NN,fig:rob_cls} shows the adversarial robustness of different models as a function of the maximal allowed translation along with the lower bound predicted by our analysis, and \cref{fig:cons_cls,fig:cons_1NN} shows the same in terms of consistency. For all range of translations, our method is the most robust. This is true both when adversarial robustness and consistency are evaluated using the index of the first nearest neighbor or when they are evaluated using the class.

\textbf{Cyclic Results:} As stated in Theorem 2, our adjusted method achieves 100\% robustness to cyclic translations. ~\cref{tab:cyclic} shows that applying our method our method as preprocess barely influences Dinov2 accuracy, which is 83.9\% using RICS-MH and 84.1\% without it in our experiments. These results shows that our method has higher accuracy than all previous methods solving the 100\% consistency task.

\begin{table}[htbp]
  \centering
    \resizebox{0.7\textwidth}{!}{
  \begin{tabular}{llcc}
    \toprule
    \textbf{Method} & \textbf{Base Model} & \textbf{Accuracy (\%)} & \textbf{Consistency (\%)} \\
    \midrule
    APS \cite{chaman2021truly} & ResNet-18 & 67.6 & 100 \\
    LPS \cite{rojas2022learnable} & ResNet-18 & 69.11 & 100 \\
    A-SwinV2-T \cite{rojas2023making} & MViTv2 & 79.91 & 99.98 \\
    A-MViTv2-T \cite{rojas2023making} & MViTv2 & 77.46 & 100 \\
    AFC \cite{michaeli2023aliasfree} & Convnext-tiny & 82.11 & 100 \\
    \midrule
    RICS-MH & Convnext-Large & 81.83 & 100 \\
    RICS-rand & DINO-v2 & 83.77 & 100 \\
    \textbf{RICS-MH} & \textbf{DINO-v2} & \textbf{83.93} & \textbf{100} \\
    \bottomrule
  \end{tabular}
  }
  \caption{Accuracy and Consistency to integer shifts for cyclic shits of our method verses four previus works. All models are evaluated on the ImageNet validation set (50K images), The results of the previous works in this table are copied from the papers. All methods achieves 100\% consistency, but our is the only one that can use any pretrained model with requires no further training, and using DINOv2 achieves the best accuracy while maintaining 100\% consistency. }
  
  \label{tab:cyclic}
  
\end{table}

Additional results and figures are attached in the appendices.

\label{sec:discussion}
\section{Limitations and Extensions}
The main limitation of our theoretical analysis and our proposed algorithm is that we can only guarantee robustness to small realistic image translations: the bound decreases with the size of the allowed translation. In contrast, when dealing with cyclical translations our method, as well as previous methods, can guarantee 100\% robustness to {\em any} circular translation. In order to extend our theory and our algorithm to larger translations, we need to tackle the difficult problem of going beyond a pseudo-random score function and towards score functions that are likely to capture the meaningful parts of the image. 

One possible approach towards meaningful score functions is to note that the RICS method with a score function that is based on dot-products can be differentiated with respect to the kernel. We can therefore try to improve our method by introducing a consistency loss that is applied on pairs of images and its translations. Note this this method would still not be computationally demanding and require only a single dot-product per crop. 

Another limitation of our method is that it can only handle integer translations. We believe that it can also be extended to fractional translations by appropriately upsampling the image prior to computing the crops, but we leave this investigation for future work. 

\section{Discussion}
The performance of deep neural networks in image classification has been continuously improving and one major factor allowing for this improvement has been the use of larger and larger datasets. This has led to the belief that the major failure modes of deep classifiers (i.e. their brittleness and inability to generalize to scenarios that are very different from the training set) will mostly be solved by using even larger datasets or more augmentations. In this paper, we have focused on one example of brittleness: the inability of modern classifier to maintain consistency when subjected to a realistic translation of a single pixel.
We have shown that methods that use sampling theory to provably provide robustness to cyclic shifts, don't solve the problem of realistic shifts (where a pixel exits the field of view, another enters on the other side). We have also shown that even classifiers that are trained on billions of training images do not learn to be robust to these trivial transformations. Instead of relying on sampling theory or on large-scale datasets, we suggested a simple method that to convert any classifier into a classifier that achieves a high degree of robustness. We presented a theory that proves a lower bound on the robustness and this bound by itself can be used to convert any classifier into one that has at least 93\% robustness to single pixel shifts (and even higher, with a price of accuracy), and 100\% robustness to any cyclic shift. We also presented experiments that show that the actual performance is even better. When using DINO2 as our baseline classifier we presented the first classifier that achieves over 82\% accuracy on ImageNet and over 95\% robustness to realistic single pixel shifts and a classifier that achieves about 84\% accuracy on ImageNet and over 100\% robustness to cyclic shifts.





\newpage

\bibliographystyle{splncs04}
\bibliography{main}
\clearpage
\newpage
\section*{Lost in Translation - Appendixes}
\label{appendices}

\subsection*{Appendix 1 - More examples of DINOv2 change in predicted label}

\begin{figure*}[ht!]
  \centering
  \begin{subfigure}{0.8\linewidth}  
    \fbox{\includegraphics[width=0.95\linewidth]{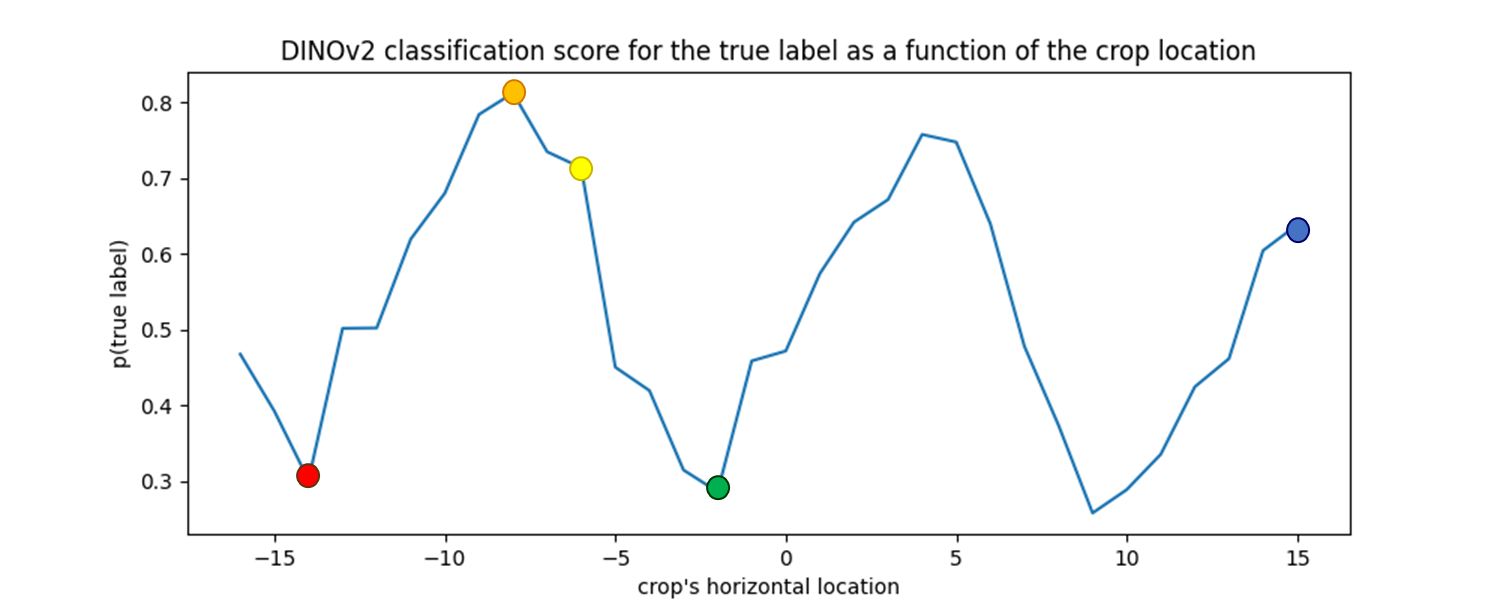}}
    \caption{DINOv2's output score for the true label ("black widow") as a function of the crop location.}
    \label{fig:widow}
  \end{subfigure}
  \hfill
  \begin{subfigure}{0.9\linewidth}
    \fbox{\includegraphics[width=0.95\linewidth]{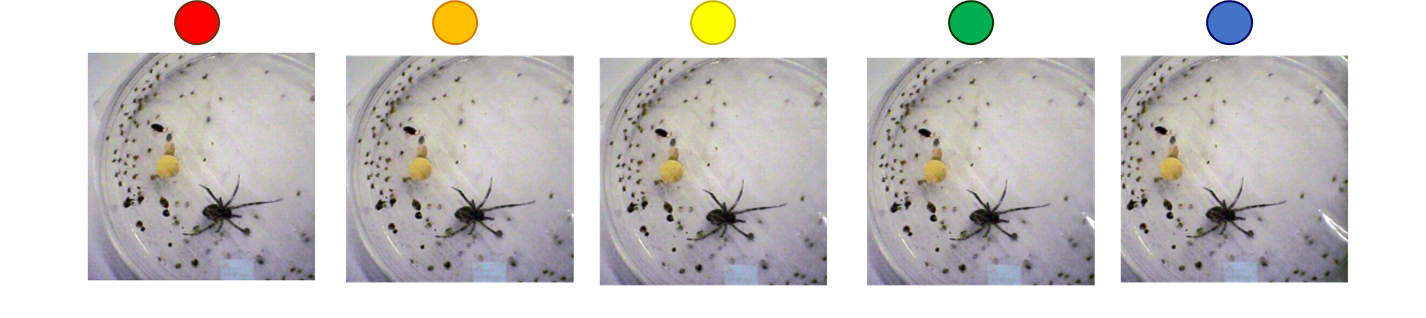}}
    \caption{Five examples of different crops, the colored circles connect the image to it's predicted probability in figure (a)}
    \label{fig:widow_images}
  \end{subfigure}
  \caption{more examples of  the paper's figure 1 showing that DINOv2 suffers from big changes in the output probability for the true label ("black widow") as a function of minor realistic translations. One can also note the clear osculations in the output probability, which occurs as a result of the use of ViT, with non overlapping patches. The translation here is horizontal.}
  \label{fig:app01_1}
\end{figure*}

\begin{figure*}[ht!]
  \centering
  \begin{subfigure}{0.7\linewidth}  
    \fbox{\includegraphics[width=0.95\linewidth]{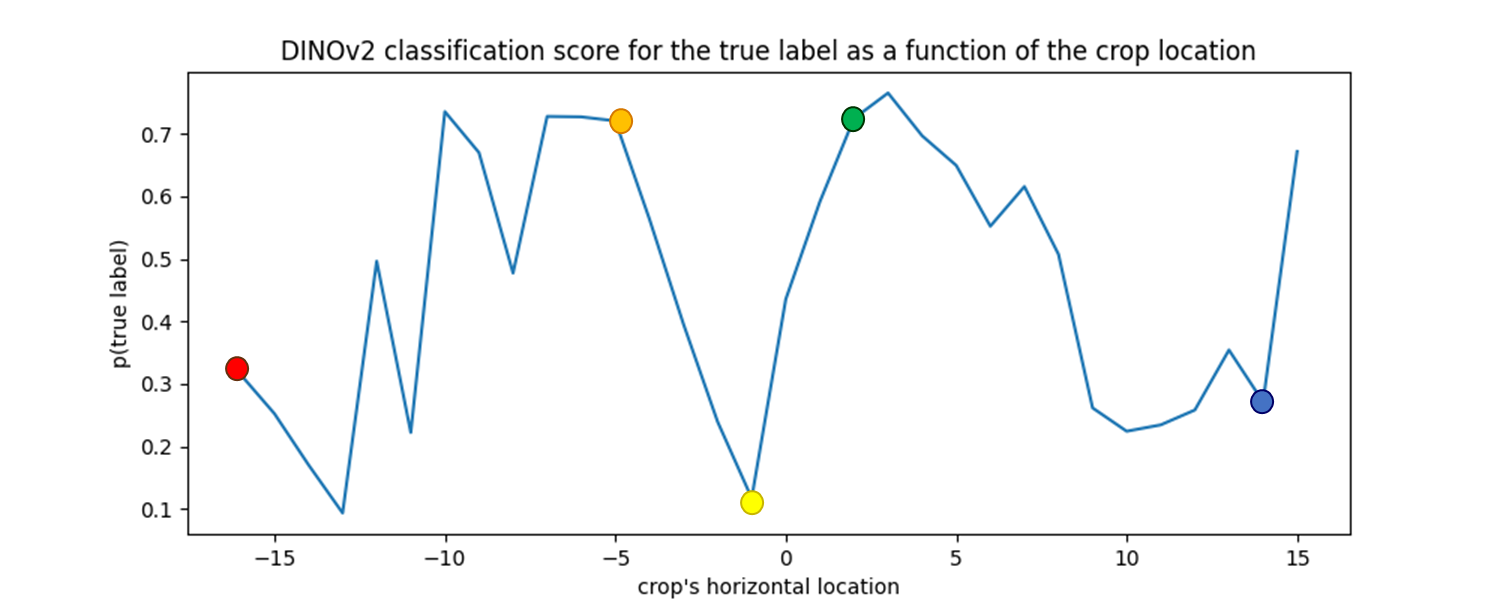}}
    \caption{DINOv2's output score for the true label ("water bottle") as a function of the crop location.}
    \label{fig:bottle_ball}
  \end{subfigure}
  \hfill
  \begin{subfigure}{0.75\linewidth}
    \fbox{\includegraphics[width=0.95\linewidth]{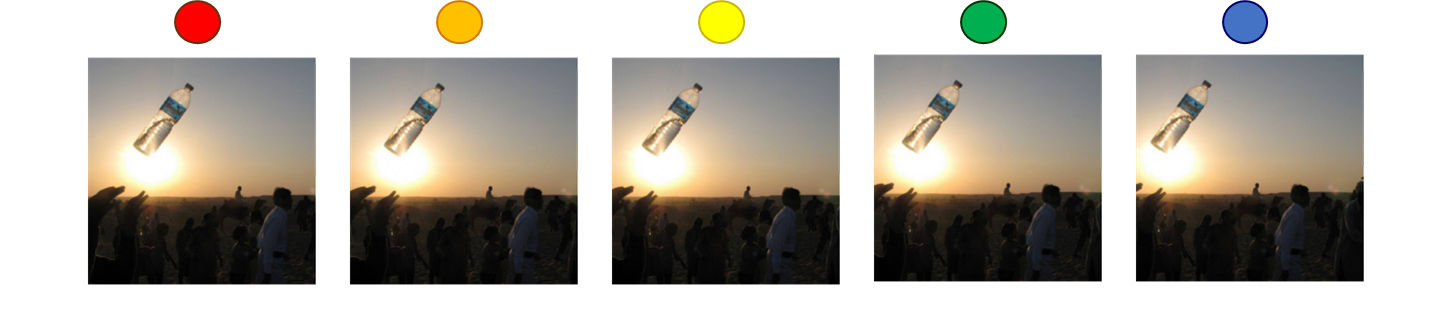}}
    \caption{Five examples of different crops, the colored circles connect the image to it's predicted probability in figure (a)}
    \label{fig:bottle_images}
  \end{subfigure}
  \caption{more examples of  the paper's figure 1 showing that DINOv2 suffers from big changes in the output probability for the true label ("water bottle") as a function of minor realistic translations. One can also note the clear osculations in the output probability, which occurs as a result of the use of ViT, with non overlapping patches. The translation here is horizontal.}
  \label{fig:app01_2}
\end{figure*}


\begin{figure*}[ht!]
  \centering
  \begin{subfigure}{0.7\linewidth}  
    \fbox{\includegraphics[width=0.95\linewidth]{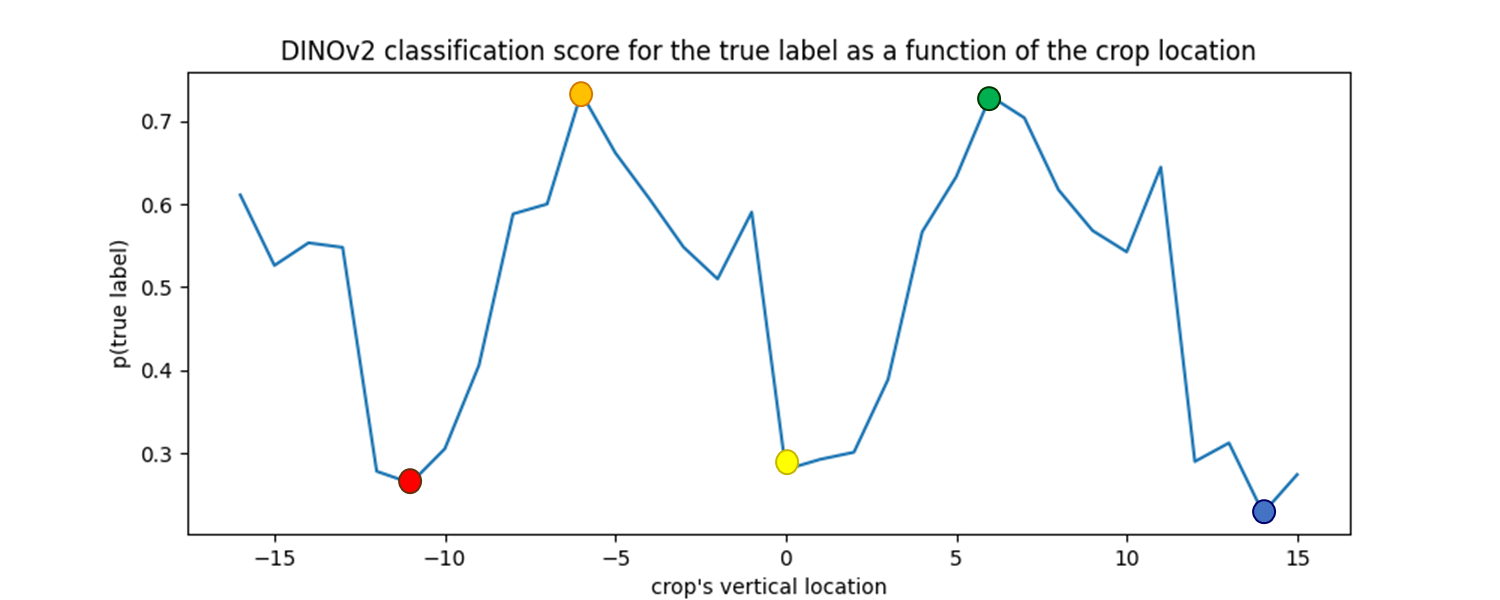}}
    \caption{DINOv2's output score for the true label ("Yorkshire terrier") as a function of the crop location.}
    \label{fig:terrier}
  \end{subfigure}
  \hfill
  \begin{subfigure}{0.75\linewidth}
    \fbox{\includegraphics[width=0.95\linewidth]{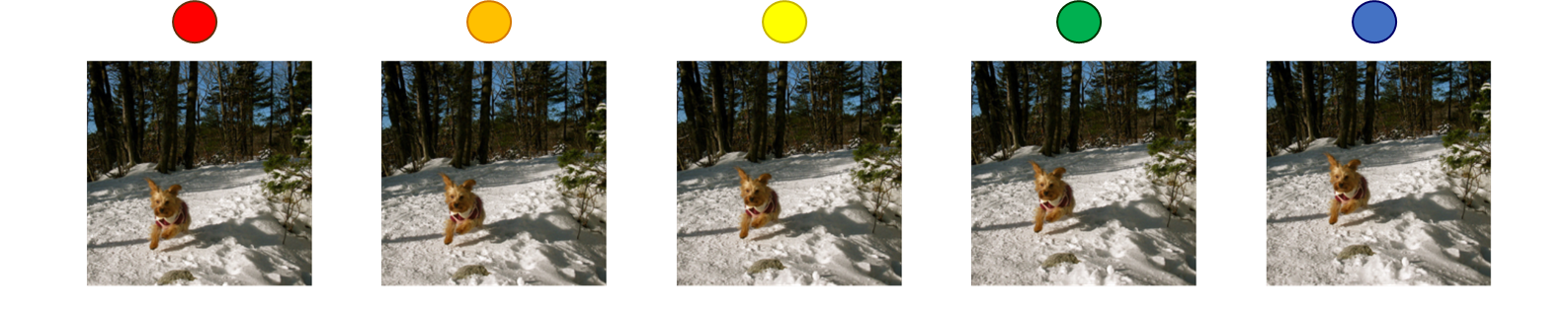}}
    \caption{Five examples of different crops, the colored circles connect the image to it's predicted probability in figure (a)}
    \label{fig:terrier_images}
  \end{subfigure}
  \caption{more examples of  the paper's figure 1 showing that DINOv2 suffers from big changes in the output probability for the true label ("Yorkshire terrier") as a function of minor realistic translations. One can also note the clear osculations in the output probability, which occurs as a result of the use of ViT, with non overlapping patches. The translation here is vertical.}
  \label{fig:app01_3}
\end{figure*}

\newpage
\clearpage
\subsection*{Appendix 2 - More examples and analyses of DINOv2 change in predicted Nearest neighbor }

\begin{figure*}[!h]
  \centering
  \begin{subfigure}{0.49\linewidth}  
    \fbox{\includegraphics[width=1\linewidth]{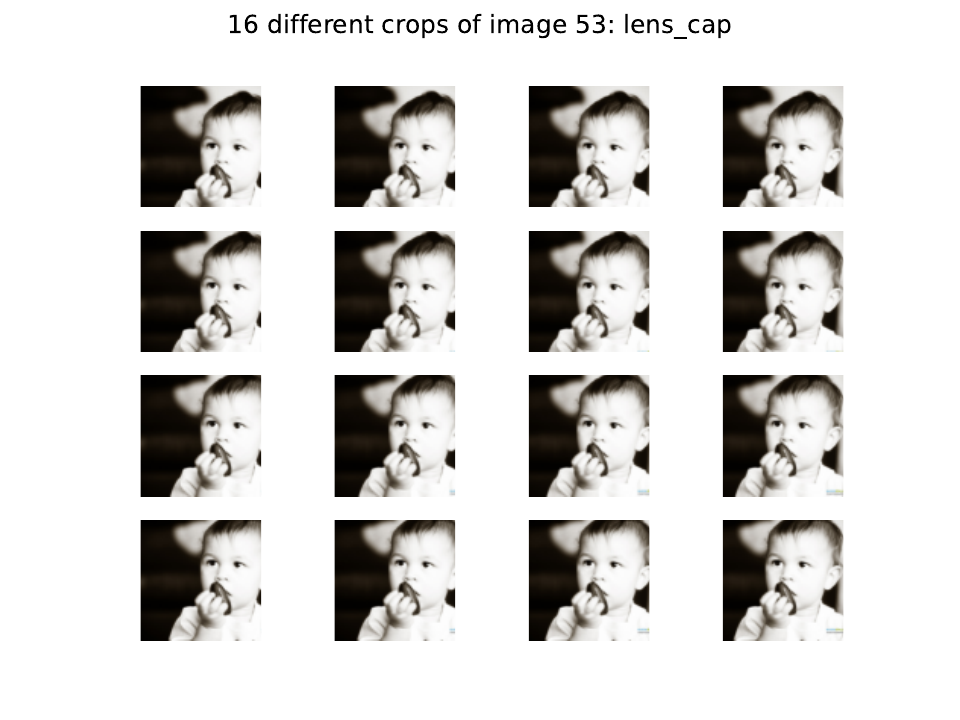}}
    \caption{Sixteen out of the possible $32 \times 32 = 1024$ images of size $224 \times 224$ pixels cropped from a $256 \times 256$ image. The main object (lens cap) is present in all the images. }
    \label{fig:image_53}
  \end{subfigure}
  \hfill
  \begin{subfigure}{0.49\linewidth}  
    \fbox{\includegraphics[width=1\linewidth]{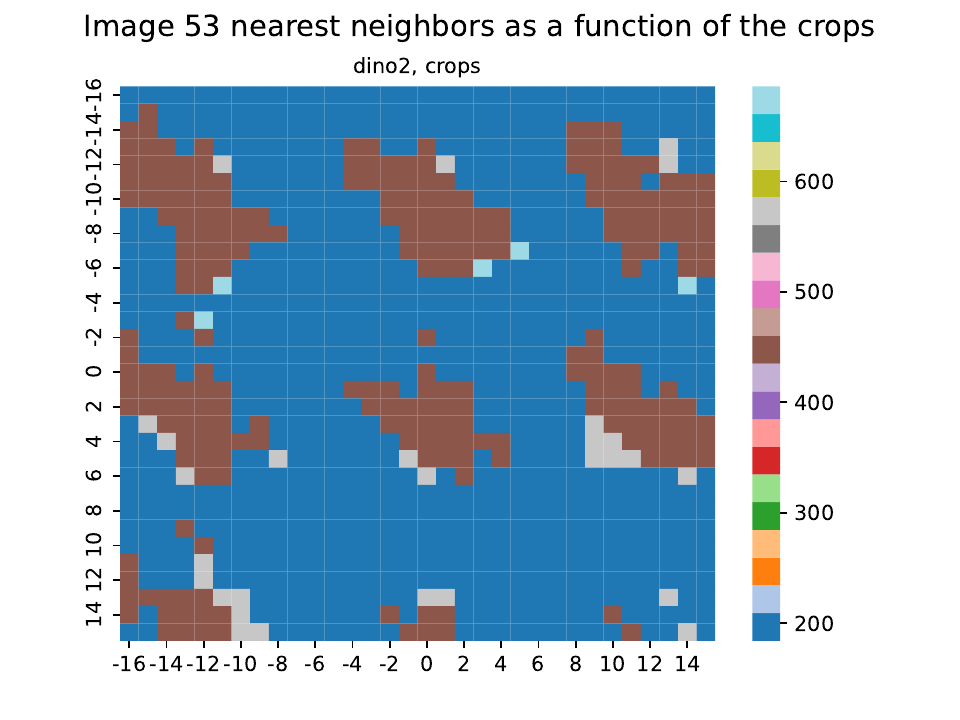}}
    \caption{The index of the image in the training set, which is the nearest neighbor of each of the $32 \times 32$ crops, ordered in a grid representing the location of the crop relative to the center crop.}
    \label{fig:image_53_posterior}
  \end{subfigure}

  \begin{subfigure}{1\linewidth}
    \fbox{\includegraphics[width=0.95\linewidth]{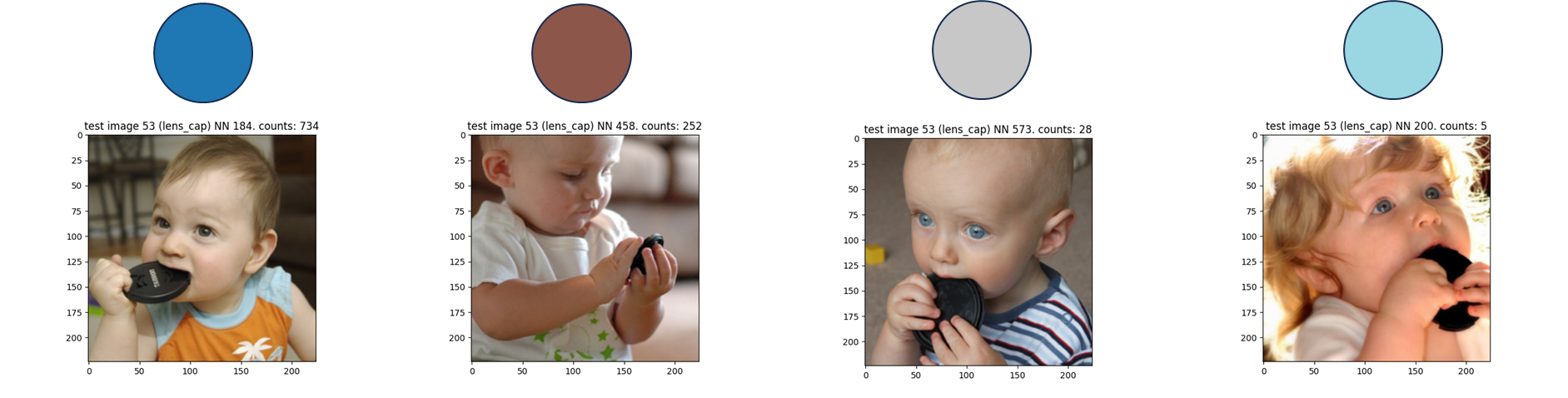}}
    \caption{The four different training images retrieved as the nearest neighbors for some crop of the image. The colored circles represent the corresponding color in Figure \ref{fig:image_53_posterior}.}
    \label{fig:NNs53}
  \end{subfigure}

  \begin{subfigure}{1\linewidth}
    \fbox{\includegraphics[width=0.95\linewidth]{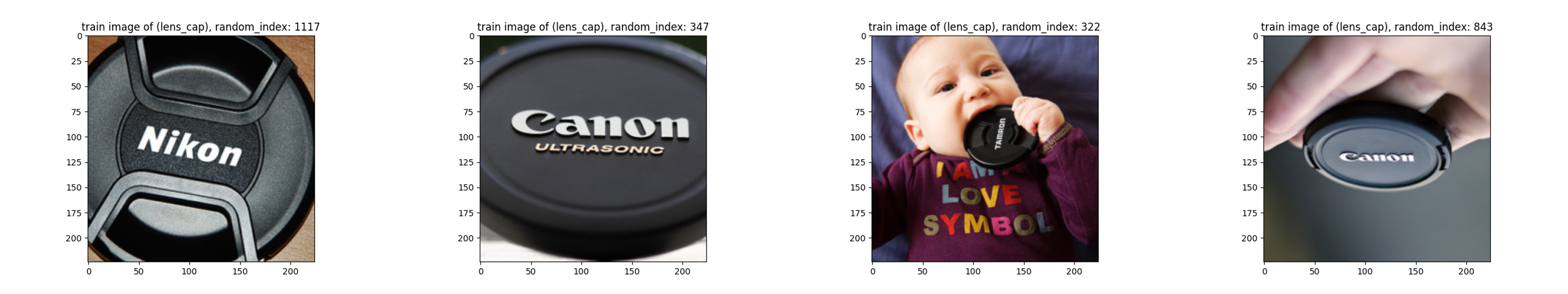}}
    \caption{Four random images of a lens cap as a reference.}
    \label{fig:rand53}
  \end{subfigure}
  
  \caption{For every test image, we calculate DINOv2's representation for each of the $32 \times 32 = 1024$ different crops. While (a) shows a monotonous continuous translation of the cropped image in both dimensions, (b) shows that the nearest neighbor alternates between several different images as shown in (c). (d) serves as a reference, illustrating that while DINOv2's representation isn't robust, it is meaningful in the sense that the nearest neighbors are similar to the input image (a baby with the lens cap).}
  \label{fig:demo}
\end{figure*}

\begin{figure*}
  \centering
  \begin{subfigure}{0.49\linewidth}  
    \fbox{\includegraphics[width=1\linewidth]{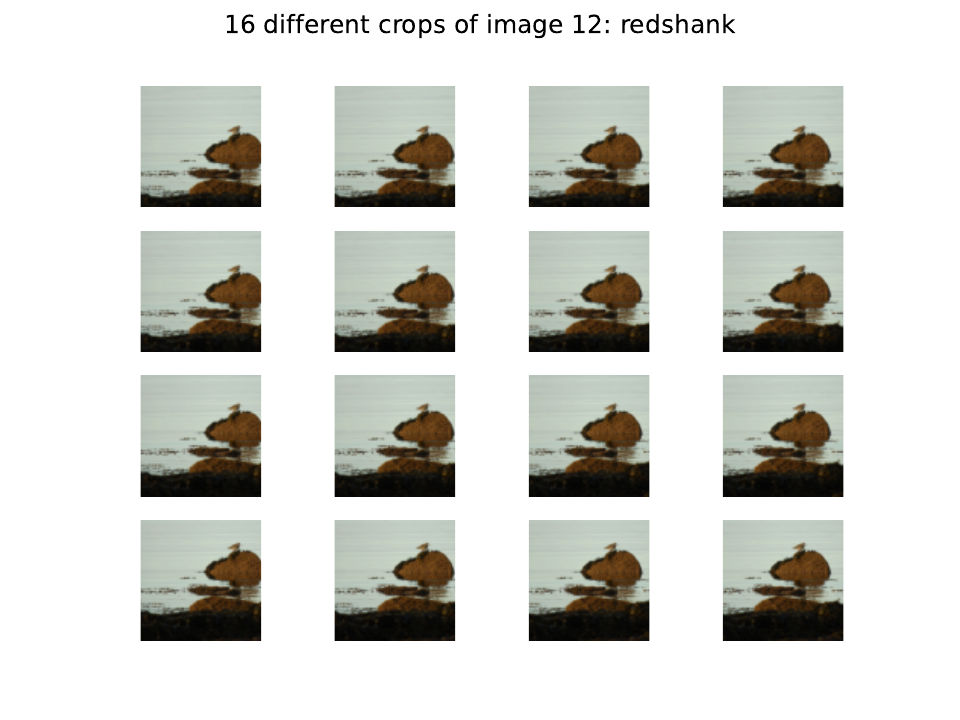}}
    \caption{Sixteen out of the possible $32 \times 32 = 1024$ images of size $224 \times 224$ pixels cropped from a $256 \times 256$ image. The main object (a redshank bird) is present in all the images. }
    \label{fig:image_12}
  \end{subfigure}
  \hfill
  \begin{subfigure}{0.49\linewidth}  
    \fbox{\includegraphics[width=1\linewidth]{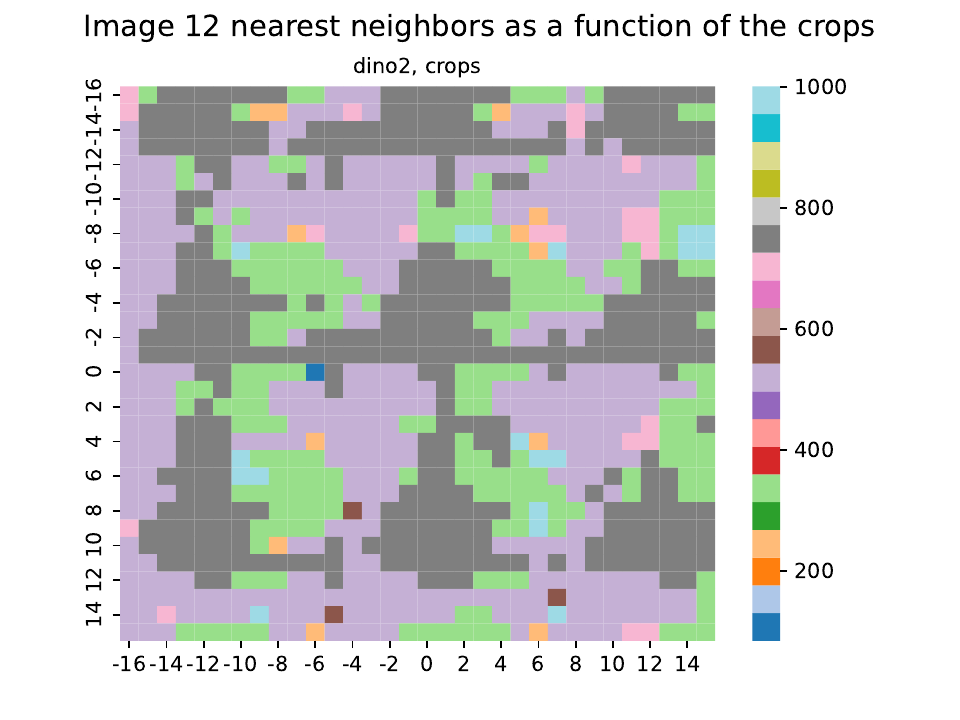}}
    \caption{The index of the image in the training set, which is the nearest neighbor of each of the $32 \times 32$ crops, ordered in a grid representing the location of the crop relative to the center crop.}
    \label{fig:image_12_NNs_map}
  \end{subfigure}

  \begin{subfigure}{1\linewidth}
    \fbox{\includegraphics[width=0.95\linewidth]{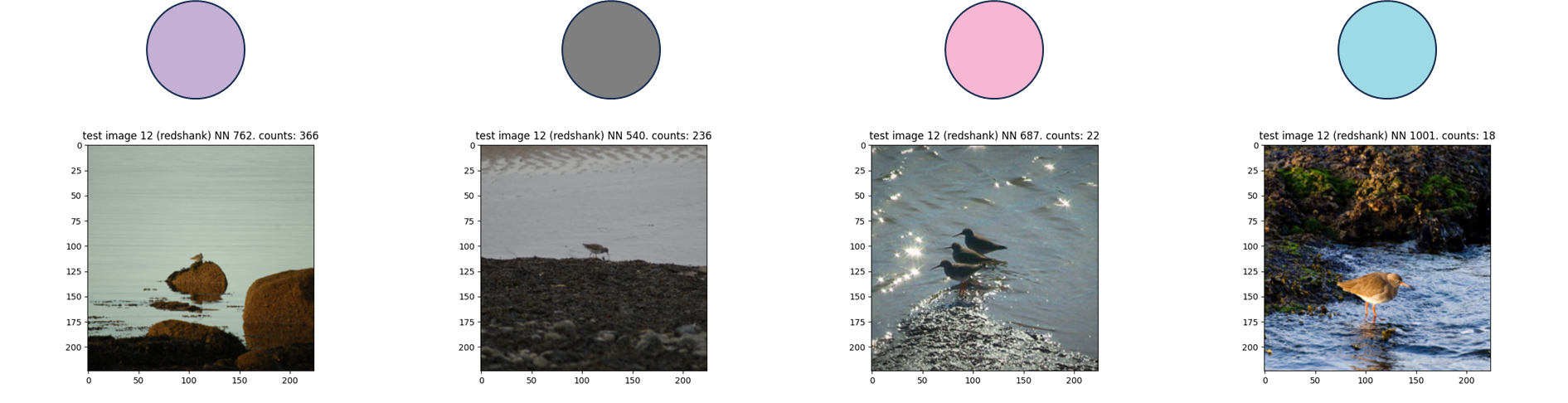}}
    \caption{The four different training images retrieved as the nearest neighbors for some crop of the image. The colored circles represent the corresponding color in Figure \ref{fig:image_12_NNs_map}.}
    \label{fig:NNs12}
  \end{subfigure}

  \begin{subfigure}{1\linewidth}
    \fbox{\includegraphics[width=0.95\linewidth]{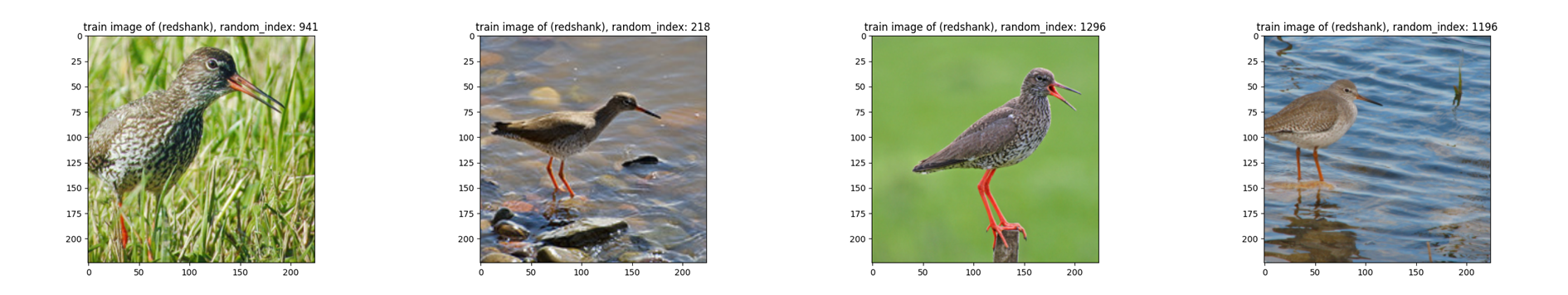}}
    \caption{Four random images of a redshank as a reference.}
    \label{fig:rand12}
  \end{subfigure}
  
  \caption{For every test image, we calculate DINOv2's representation for each of the $32 \times 32 = 1024$ different crops. While (a) shows a monotonous continuous translation of the cropped image in both dimensions, (b) shows that the nearest neighbor alternates between several different images as shown in (c). (d) serves as a reference, illustrating that while DINOv2's representation isn't robust, it is meaningful in the sense that the nearest neighbors are similar to the input image (a redshank bird).}
  \label{fig:app1}
\end{figure*}
\clearpage
\newpage
\subsection*{Appendix 3 - More examples of DINOv2 ond Open-Clip change in predicted nearest neighbor}
\begin{figure*}[ht!]
  \begin{subfigure}{0.45\linewidth}  
    \fbox{\includegraphics[width=1\linewidth]{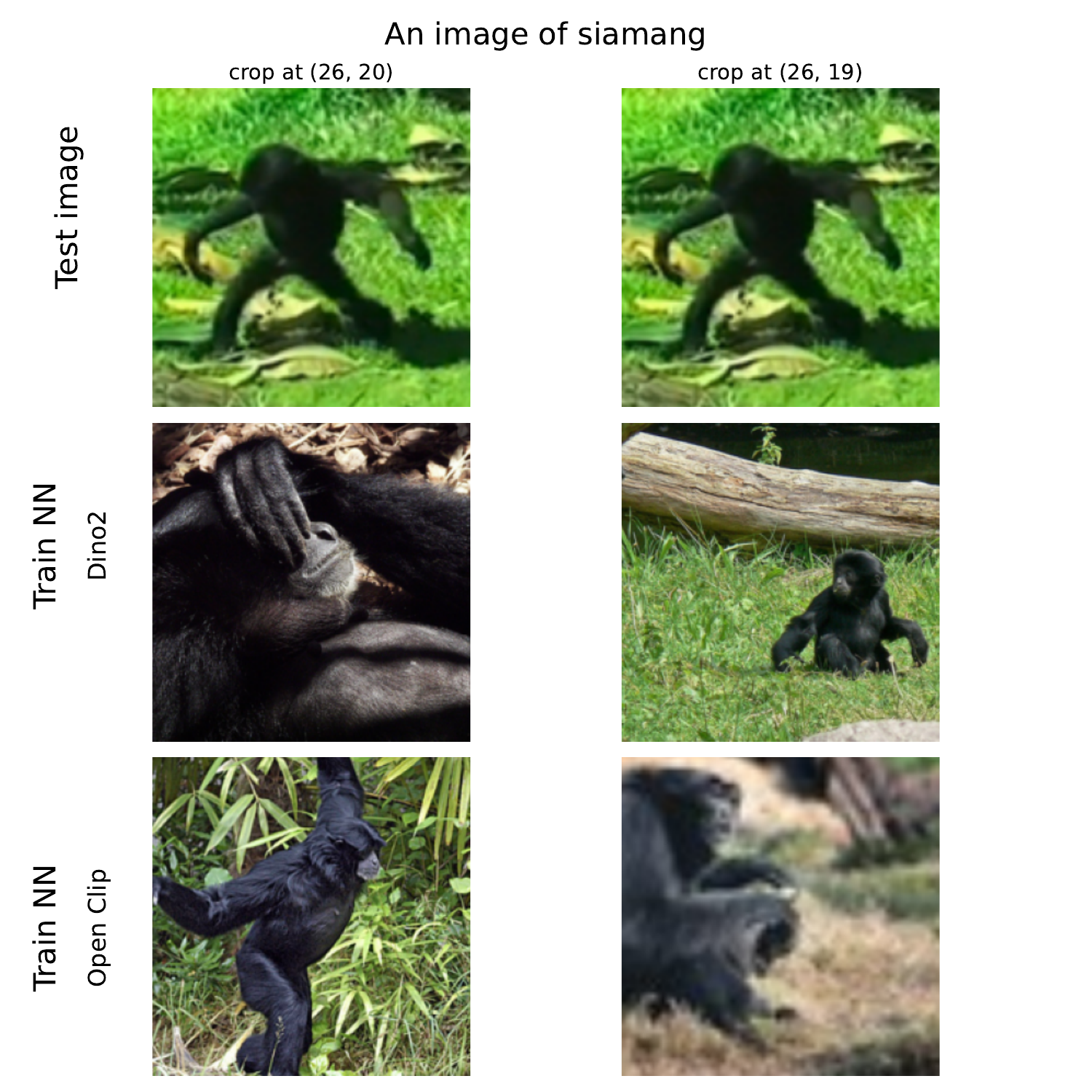}}
    \caption{}
    \label{fig:image_35}
  \end{subfigure}
    \begin{subfigure}{0.45\linewidth}  
    \fbox{\includegraphics[width=1\linewidth]{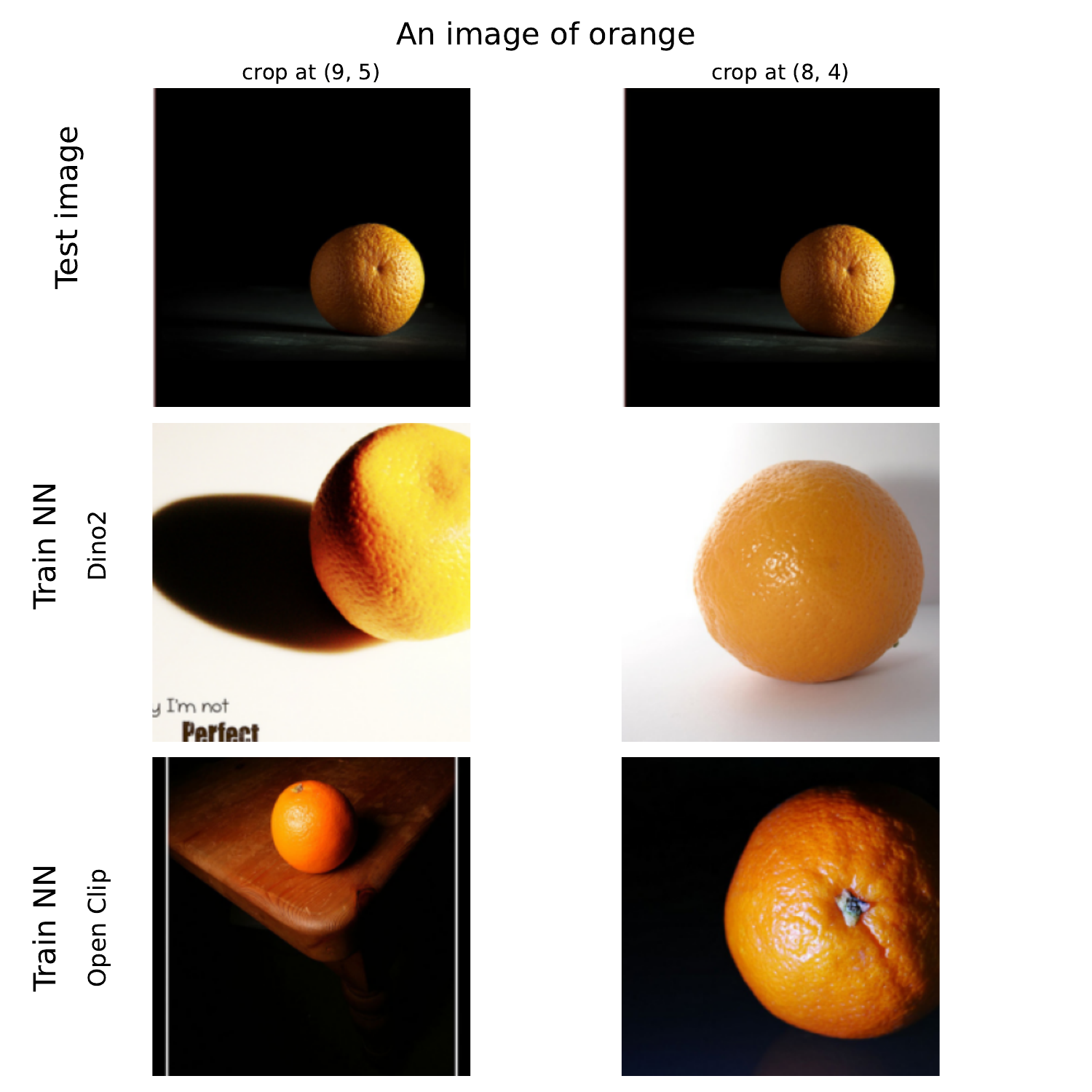}}
    \caption{}
    \label{fig:image_36}
  \end{subfigure}
  \hfill
  \begin{subfigure}{0.45\linewidth}  
    \fbox{\includegraphics[width=1\linewidth]{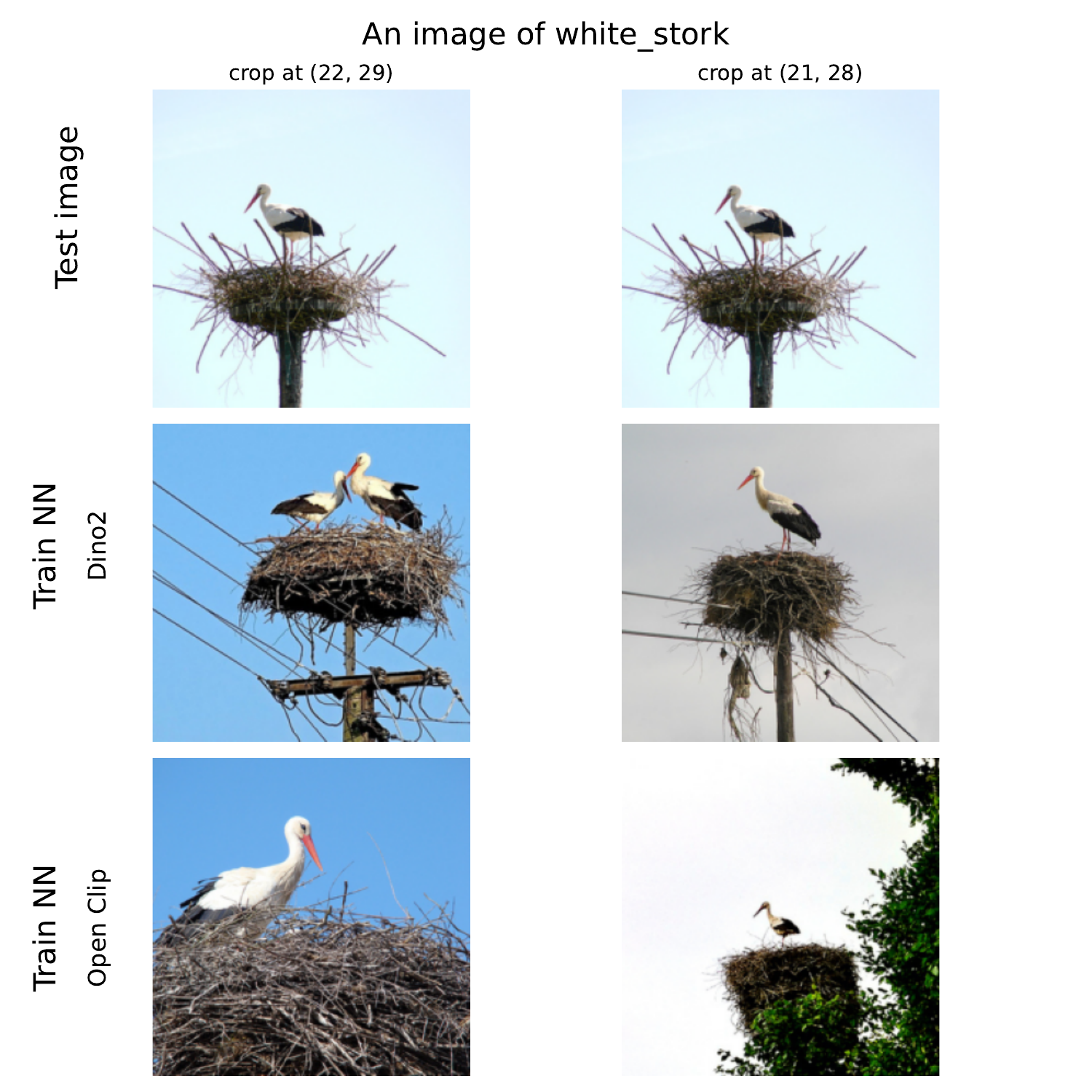}}
    \caption{}
    \label{fig:image_59}
  \end{subfigure}
    \hfill
  \begin{subfigure}{0.45\linewidth}  
    \fbox{\includegraphics[width=1\linewidth]{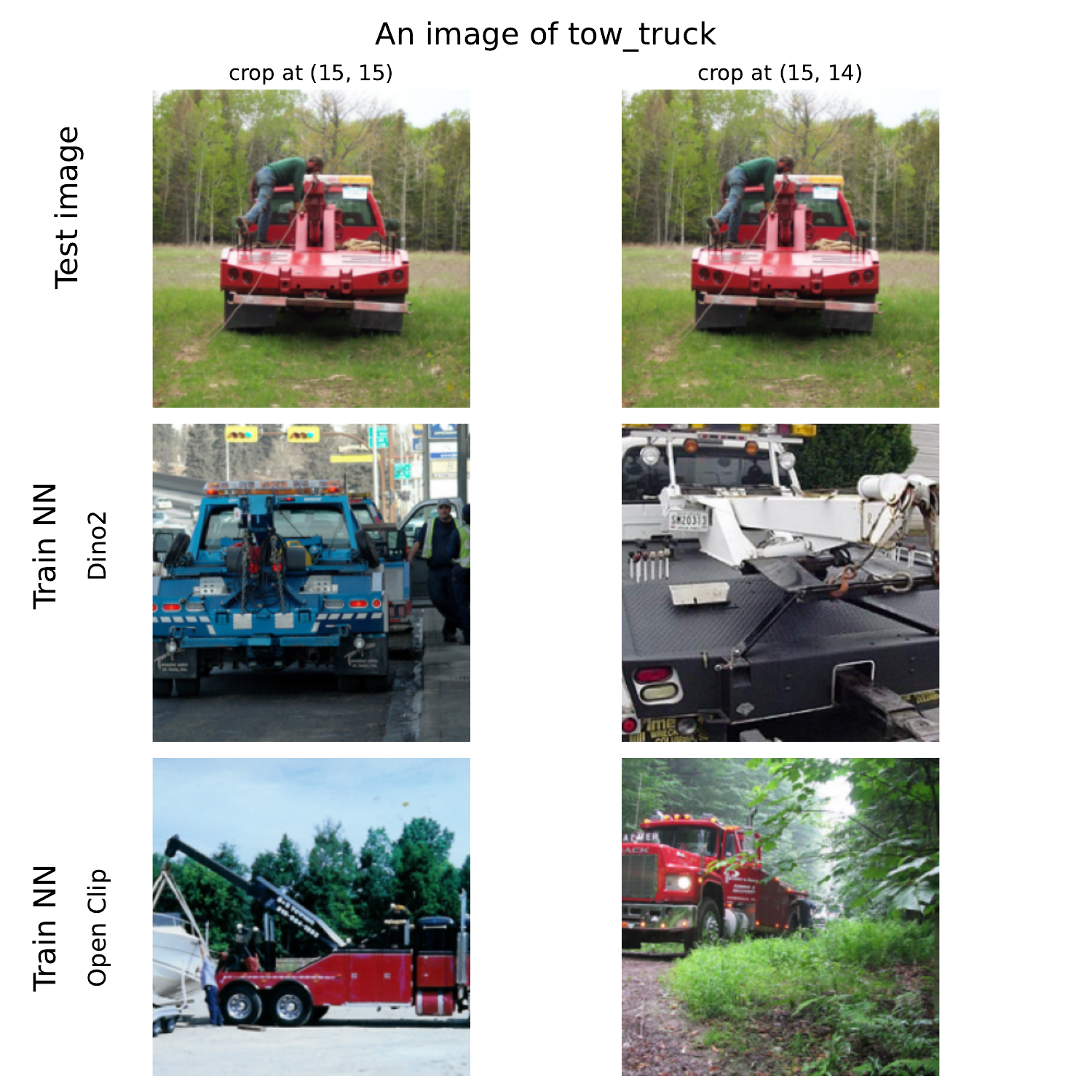}}
        \caption{}
    \label{fig:image_96}
  \end{subfigure}
      \caption{More examples of the paper's figure 2 - a change in the nearest neighbor of both DINOv2 and Open-Clip as a function of a single pixel translation. The main object is clearly seen in all of the test images}

\end{figure*}
\clearpage
\newpage

\subsection*{Appendix 4 - More examples of AFC change in predicted nearest neighbor with non-cyclic translations}

\begin{figure*}
  \centering
  \begin{subfigure}{0.75\linewidth}
    \fbox{\includegraphics[width=0.65\linewidth]{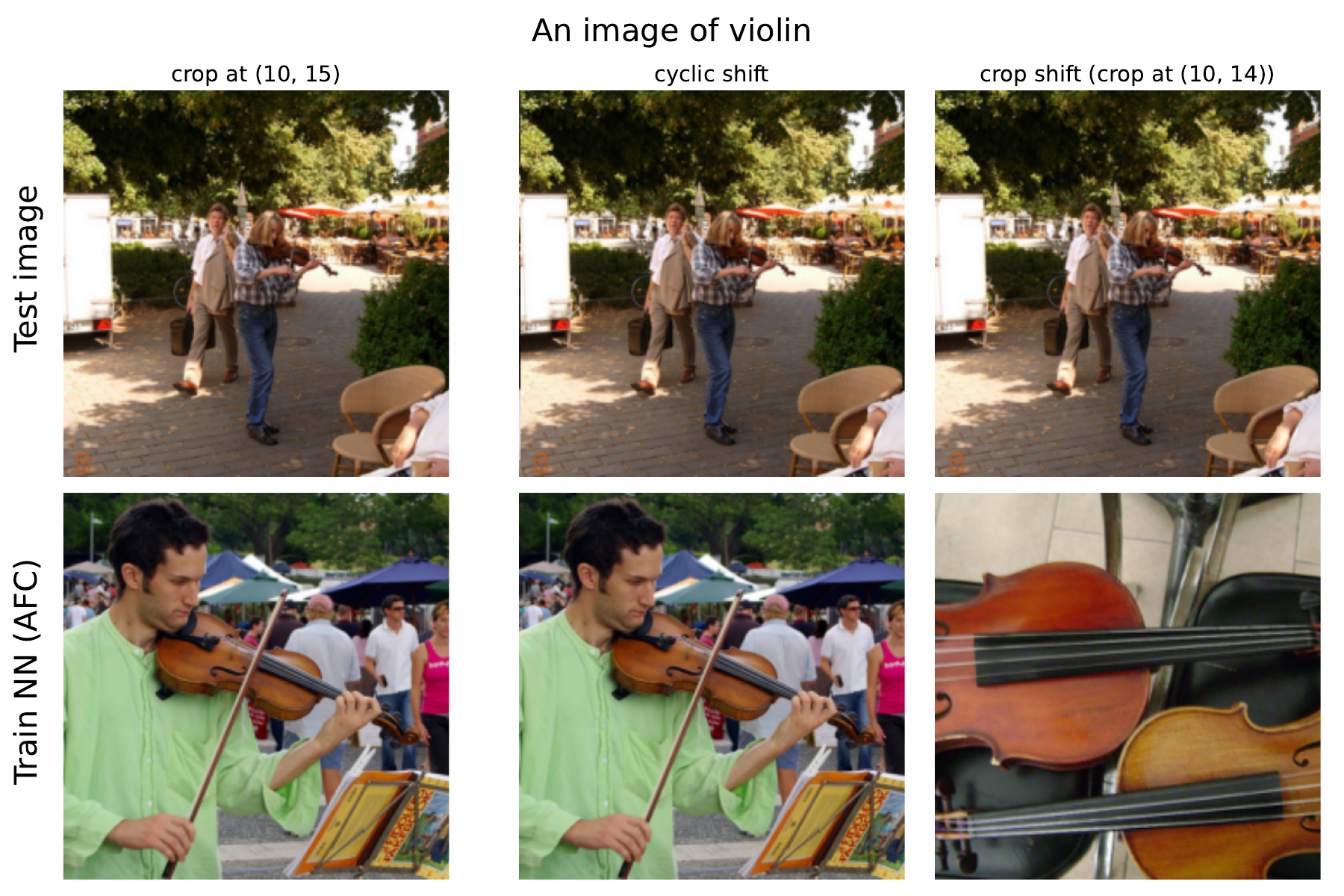}}
    \caption{}
    \label{fig:image_1}
  \end{subfigure}
  \hfill
  \begin{subfigure}{0.75\linewidth}  
    \fbox{\includegraphics[width=0.7\linewidth]{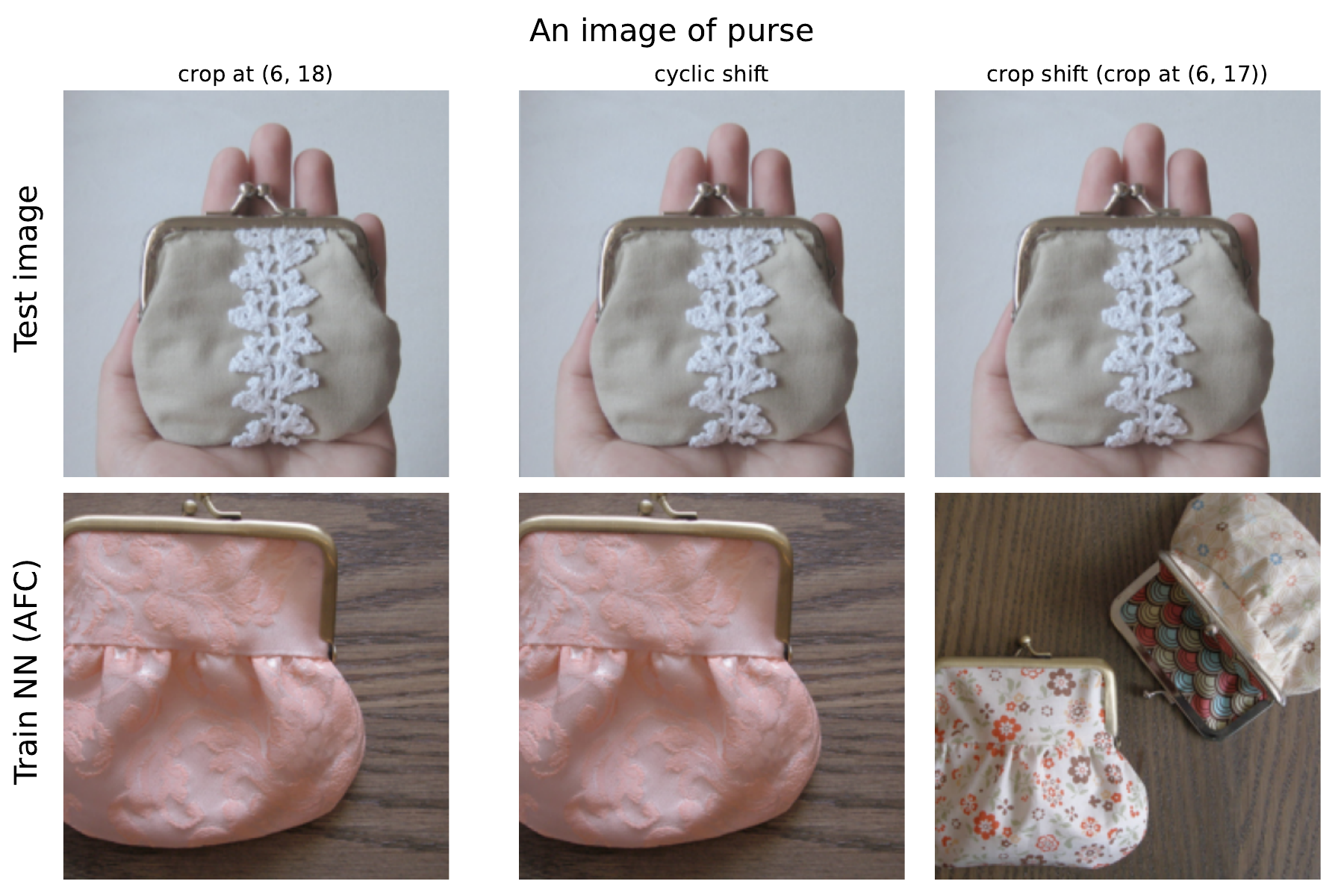}}
    \caption{}
    \label{fig:image_46}
  \end{subfigure}
    \begin{subfigure}{0.75\linewidth}  
    \fbox{\includegraphics[width=0.7\linewidth]{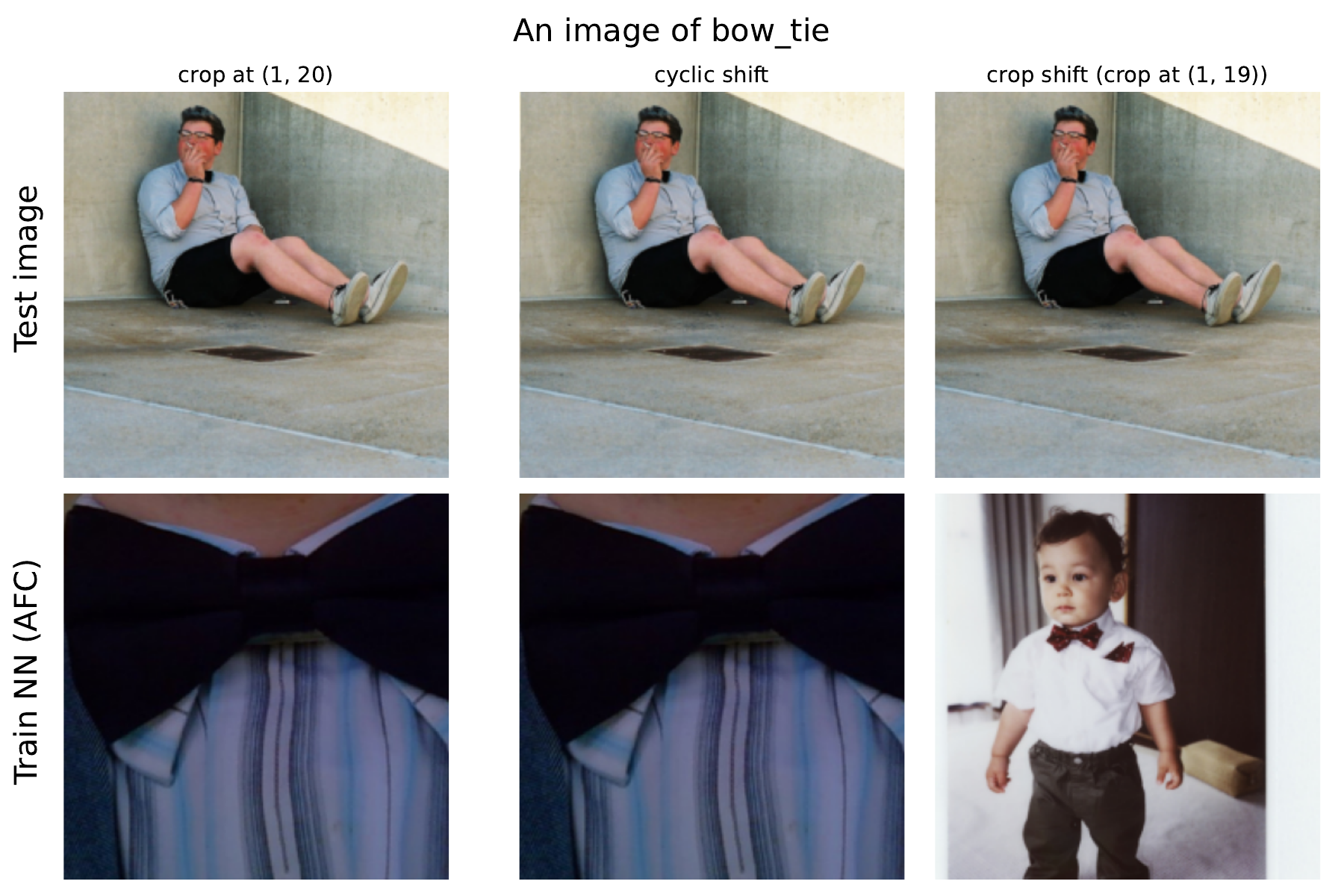}}
    \caption{}
    \label{fig:image_88}
  \end{subfigure}

\end{figure*}
\clearpage
\newpage
\subsection*{Appendix 5 - ImageNet-A Adversarial Robustness}

\begin{figure*}
  \centering
    \fbox{\includegraphics[width=0.8\linewidth]{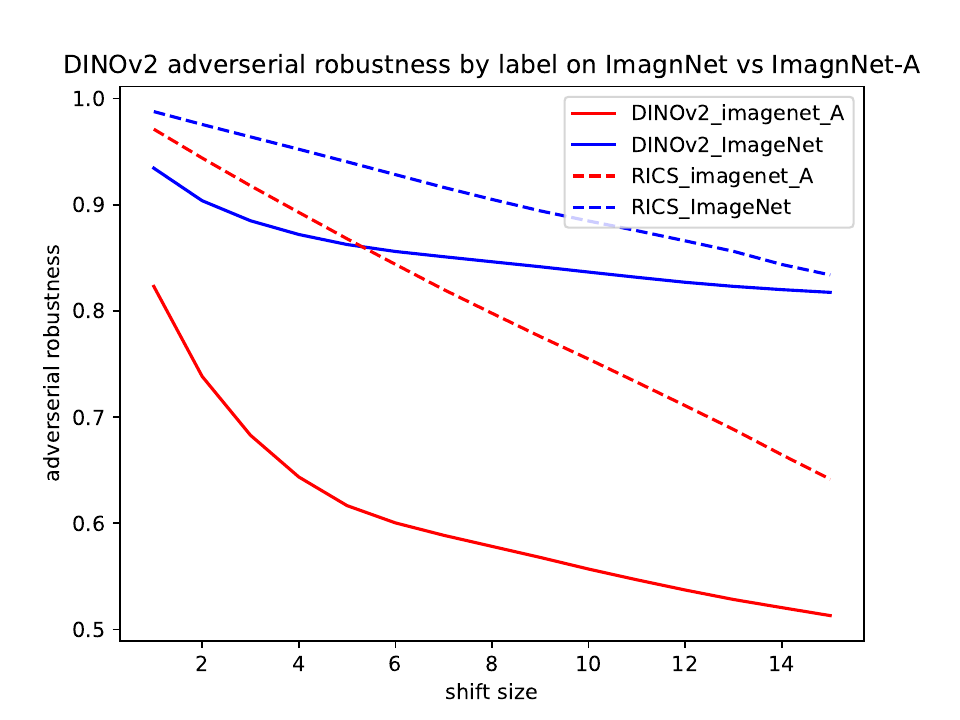}}
    \caption{When evaluated by predicted label, the adversarial robustness is somewhat correlated the the accuracy, therefor, when evaluating on the hard dataset of ImageNet-A (Hendrycks et al.) we can wee DINO is even less robust. RICS improves the robustness, but it does decline as the translations becomes bigger as the theory suggests.}
    \label{fig:app_05}

\end{figure*}
\clearpage
\newpage 
\subsection*{Appendix 6 - Robust Inferance by Crop Selection with DINOv2 results on the PLACES365 Dataset}

As stated in the paper, our method is proven to convert any classifier to a robust classifier. We have shown it's true for different models, and here we show empirically it's true for different dataset as well. We evaluate DINOv2 \cite{oquab2023dinov2} with and without RICS on the PLACES365 \cite{zhou2017places} dataset which is more diverse and challenging task than ImageNet \cite{krizhevsky2012imagenet}.

\begin{figure}[htb!]
  \centering
  \begin{subfigure}{0.5\linewidth}
    \centering
    \includegraphics[width=\linewidth]{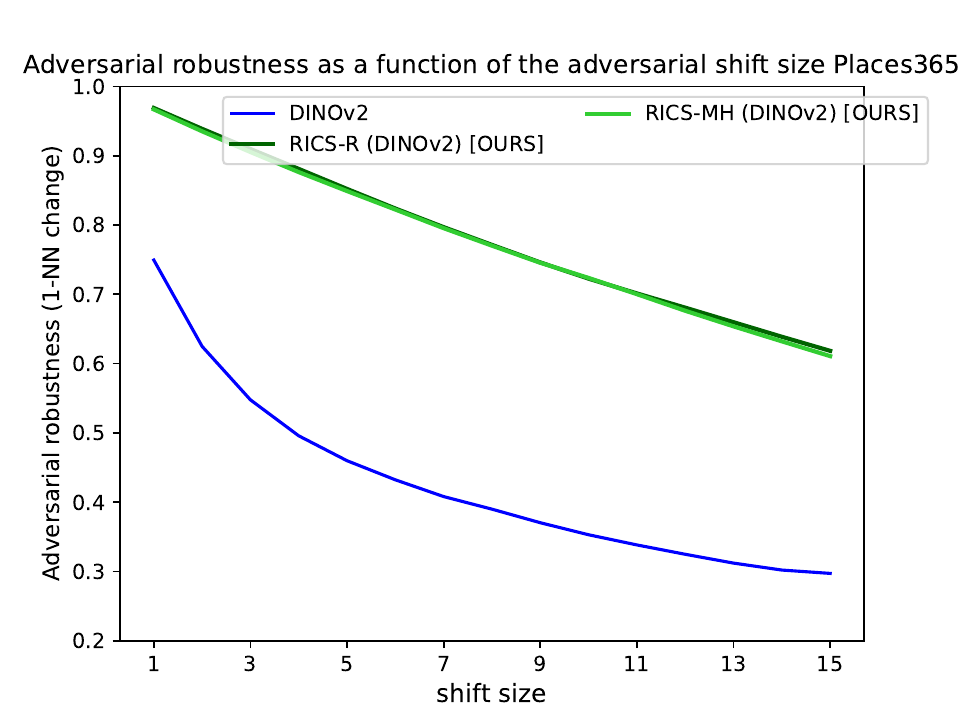}
    \label{subfig:app_6_1}
  \end{subfigure}\hfill
  \begin{subfigure}{0.5\linewidth}
    \centering
    \includegraphics[width=\linewidth]{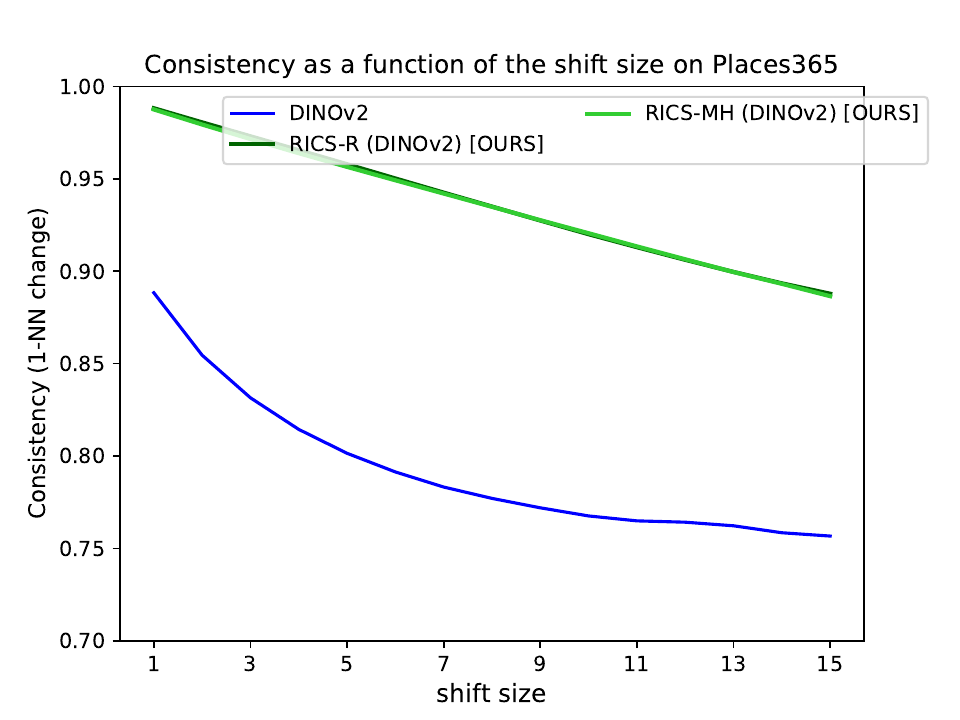}
    \label{subfig:acc_rob_rebut}
  \end{subfigure}
  \caption{Adversarial robustness (left) and Consistency (right) of the predicted nearest neighbor on the Places365 dataset. Using RICS improves the robustness dramatically in both metrics, for any shift size.}
  \label{fig:app_6_2}
\end{figure}

\begin{figure}[htb!]
  \centering
    \includegraphics[width=0.75\linewidth]{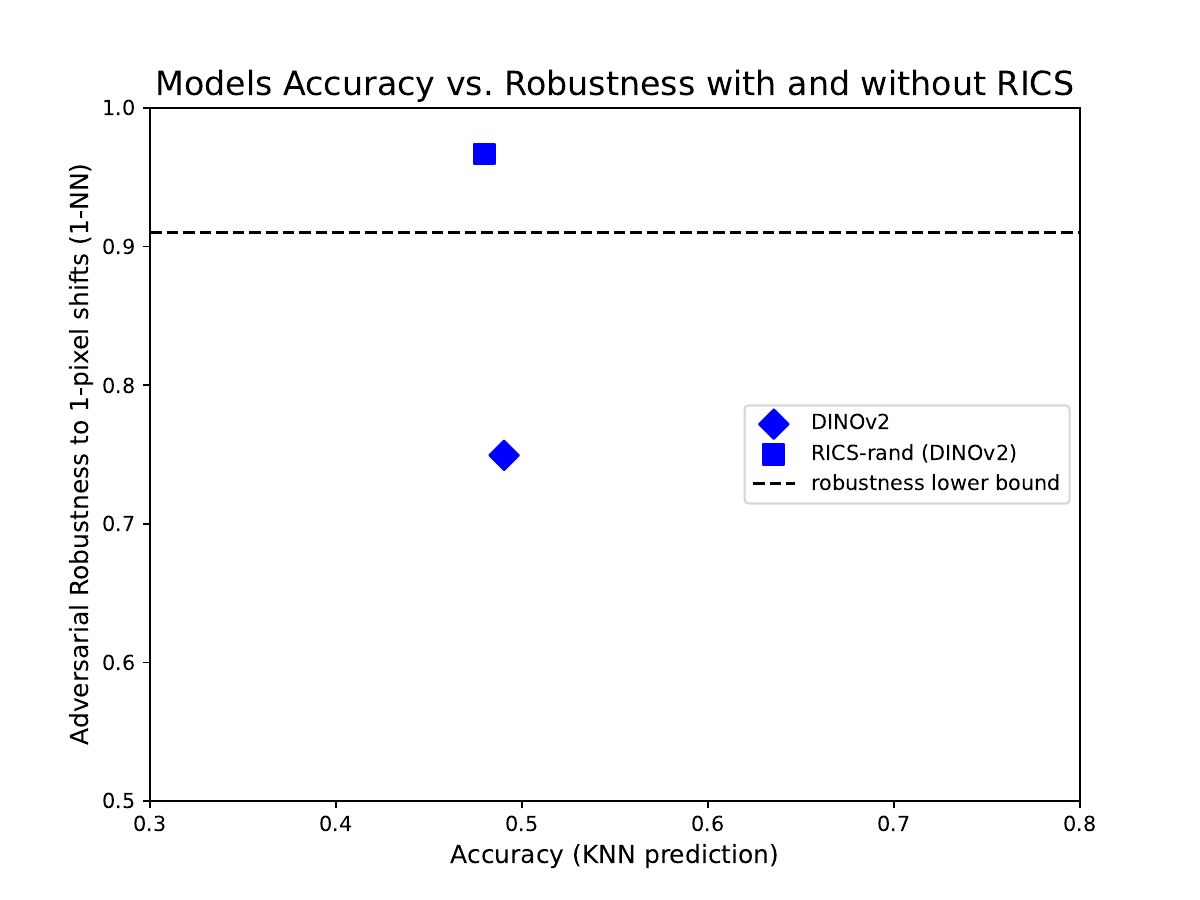}
    \caption{Adversarial robustness of the predicted nearest neighbor for 1 pixel shift and model accuracy with KNN (k=11) on the Places365 dataset. Using RICS-rand improves the robustness dramatically with a 1\% drop of accuracy only.}
    \label{subfig:app_6_3}
\end{figure}

\begin{table}
  \centering
  \begin{tabular}{@{}l|c|cccc|cccc|@{}}
    \toprule
    \multicolumn{1}{c}{Model} &  \multicolumn{1}{c}{Accuracy} &  \multicolumn{4}{c}{\makecell[c]{Adv-Rob (1-NN)\\{\footnotesize Shift Size:}}} & \multicolumn{4}{c}{\makecell[c]{Consistency (1-NN)\\{\footnotesize Shift Size:}}} \\
     \cmidrule(lr){3-6} \cmidrule(lr){7-10} 
     & KNN & 1 & 3 & 5 & 9 & 1 & 3 & 5 & 9  \\
    \midrule
DINOv2 & \textbf{49.03} & 74.94 & 54.77 & 45.98 & 37.04 & 88.83 & 83.15 & 80.15 & 77.20 \\
\midrule
RICS-MH (DINOv2) & 45.61 & \textbf{96.89} & \textbf{90.96} & \textbf{85.20} & \textbf{74.60} & \textbf{98.82} & \textbf{97.30} & \textbf{95.79} & \textbf{92.74} \\
RICS-rand (DINOv2) & \textbf{47.97} & \textbf{96.68} & \textbf{90.53} & \textbf{84.90} & \textbf{74.55} & \textbf{98.75} & \textbf{97.15} & \textbf{95.65} & \textbf{92.76} \\
    \bottomrule
  \end{tabular}
  \caption{RICS is able to make DINOv2 much more robust, with a cost of 1\% - 3.5\% in accuracy. Results are calculated based on a random sample of 1000 images from the PLACES365 validation set.}
  \label{tab:app_results}
\end{table}

\end{document}